%% file: arxiv.tex
\documentclass{article}
\usepackage{iclr2027_conference,times}
\usepackage[T1]{fontenc}

\input{math_commands.tex}

\usepackage{amsmath}
\usepackage{amssymb}
\usepackage{amsfonts}
\usepackage{amsthm}
\usepackage{booktabs}
\usepackage{array}
\usepackage{float}
\usepackage{placeins}
\usepackage{mathtools}
\usepackage{microtype}
\usepackage{xcolor}
\usepackage[letterpaper,margin=1.18in]{geometry}
\usepackage{newunicodechar}
\newunicodechar{–}{\textcolor{black}{\textendash}}
\newunicodechar{’}{\textcolor{black}{\textquoteright}}
\usepackage{hyperref}
\usepackage{url}
\newcommand{\winci}[3]{{\renewcommand{\arraystretch}{1}\begin{tabular}[t]{@{}r@{}}#1\\{[#2,\,#3]}\end{tabular}}}
\newcommand{\resulttable}{\normalsize\color{black}\setlength{\tabcolsep}{3pt}\renewcommand{\arraystretch}{1.05}\setlength{\belowcaptionskip}{\baselineskip}}
\newcommand{\tablegroup}[2]{\addlinespace[3pt]\multicolumn{#1}{@{}l}{\textit{#2}}\\[2pt]}
\newcommand{\tablenote}[1]{\par\vspace{\baselineskip}\begin{minipage}{\linewidth}\normalsize #1\end{minipage}}

\theoremstyle{plain}
\newtheorem{theorem}{Theorem}[section]
\newtheorem{proposition}[theorem]{Proposition}

\theoremstyle{definition}

\theoremstyle{remark}
\newtheorem{remark}[theorem]{Remark}

\newcommand{\D}{\mathcal{D}}
\newcommand{\X}{\mathcal{X}}
\newcommand{\Y}{\mathcal{Y}}

\newcommand{\Prob}{\mathbb{P}}
\newcommand{\piref}{\pi_{\mathrm{ref}}}
\newcommand{\method}{\textsc{UNM-DPO}}
\newcommand{\methodln}{\textsc{ULNM-DPO-WR}}

\title{Uncertainty-Normalized Margins for\\
Direct Preference Optimization}

\author{%
\begin{minipage}[t]{\dimexpr\textwidth-2\tabcolsep\relax}
\raggedright
\textbf{Sadegh Khorasani\textsuperscript{1}, Petrus Mikkola\textsuperscript{2},}
\textbf{Matthias Grossglauser\textsuperscript{1}}\\[0.5em]
\normalfont\small
\textsuperscript{1}School of Computer and Communication Sciences, EPFL, Lausanne, Switzerland\\
\textsuperscript{2}University of Helsinki, Helsinki, Finland\\[0.5em]
{\footnotesize\texttt{sadegh.khorasani@epfl.ch}\quad
\texttt{petrus.mikkola@gmail.com}}\\
{\footnotesize
\texttt{matthias.grossglauser@epfl.ch}}
\end{minipage}%
}

\iclrfinalcopy
\begin{document}

\maketitle

\begin{abstract}

Direct preference optimization (DPO) models binary preferences through a Bradley–Terry model with a common noise scale, without explicitly accounting for preference strength or prompt-dependent uncertainty from human feedback. We introduce uncertainty-normalized margin DPO (\method{}), which combines strength-dependent margins with a learned prompt scale. Motivated by a heteroskedastic Bradley--Terry model, we develop two training objectives. Both compare the implicit rewards of preferred and rejected responses, derived from response log-probability ratios to a reference policy. Advantage-only (AO) divides this reward difference by the prompt scale before subtracting the margin; whole-residual (WR) subtracts the margin before dividing by the scale. For the WR comparison model, we establish a necessary and sufficient condition under which known margins make the prompt scale identifiable. We introduce a practical procedure for learning the scale. Building on WR, we introduce ULNM-DPO-WR, which normalizes each response’s implicit reward by its length. We evaluate our methods against DPO and related baselines on HelpSteer2 and HelpSteer3, using the Skywork reward model as a judge. With Llama-3.1-8B-Instruct, ULNM-DPO-WR achieves tie-adjusted win rates against matched DPO of 68.00\% and 65.31\% on evaluation panels, with higher mean rewards and shorter responses on average. \textcolor{black}{On AlpacaEval with a GPT-4.1 judge and GPT-4-Turbo reference answers, the same 8B policy achieves a length-controlled win rate of 21.62\%, compared with 16.39\% for DPO and 15.30\% for SimPO.} These results demonstrate the potential of combining preference-strength margins, learned prompt scales, and length normalization for policy optimization.

\end{abstract}

\section{Introduction}
\label{sec:introduction}

Learning from preference comparisons is a practical way to improve a
language model's responses without requiring an ideal answer for every prompt \citep{ouyang2022training}. Direct preference optimization (DPO) \citep{rafailov2023direct} simplifies this process by training a policy directly on preferred and rejected responses, avoiding a separate reward-model training stage and reinforcement learning loop. Its standard Bradley-Terry \citep{bradley1952rank} formulation uses binary preference feedback, while human annotators can often provide richer information by revealing the strength of their preference judgments.

Preference strength-aware methods use this information \citep{touvron2023llama}. For example, ODPO introduces a margin (`offset') into DPO under the Bradley-Terry model \citep{amini2024offset}, while HelpSteer2-Preference considers two margin variants of such model \citep{wang2025helpsteer2preference}. However, the uncertainty in preference feedback is inherently linked to preference strength \citep{wang2024secrets}, and how this uncertainty impacts policy optimization remains relatively unexplored. Preference uncertainty varies from prompt to prompt, since the certainty and strength of preference judgments naturally depend on the underlying prompt (e.g., unclear task specification) \citep{zhang2025ICML}. This motivates combining margins with a prompt-dependent scale, while distinguishing the roles of these two quantities in policy optimization. For a brief literature review of the topic,
see Appendix~\ref{app:related-work}.

We introduce uncertainty-normalized margin DPO (\method{}), which combines preference-strength margins with a learned prompt-dependent scale. Motivated by a Bradley-Terry model with prompt-dependent noise, we develop two training objectives: advantage-only (AO) and whole-residual (WR). Three key components are margin, prompt scale, and the implicit reward difference between preferred and rejected responses, computed from their log-probability ratios to a fixed reference policy. The AO and WR objectives differ in how they combine reward difference with the margin and scale: AO divides the reward difference by the scale before subtracting the margin, whereas WR subtracts the margin first and then divides by the scale. Thus AO lets the scale change the reward difference required to exceed the margin, while WR preserves that threshold and changes the loss's sensitivity around it.

Learning at the prompt scale is challenging due to per-prompt data scarcity, as preference datasets often contain only one comparison per prompt \citep{bai2022training,wang2025helpsteer3}, as well as model identifiability issues. To learn it, we first split training prompts into five folds and train five auxiliary policies, each on all folds except one. Each comparison's implicit reward difference is computed using out-of-fold auxiliary policy. We pool these differences and strength annotations to fit bounded prompt-scale network with a regularized loss, then freeze it before training the final policy. Because response log-probability ratios sum contributions across tokens, response length can affect both the margin and the scale fitting. Inspired by length normalization in preference optimization \citep{meng2024simpo}, we address this through
\methodln{}, a length-normalized extension of WR. It divides each
response's log-probability ratio by its token count before forming the reward difference and uses the same normalization when fitting the scale. Only the final policy is needed at inference time.

In summary, our contributions are:
\begin{itemize}
\setlength{\itemsep}{2pt}
\item We combine strength-dependent margins,
a learned prompt scale, and response-length normalization in a
reference-relative policy objective, with consistent normalization
in scale fitting and final-policy training.
\item For a fixed prompt,
we establish a necessary and sufficient condition under which known
margins identify the prompt scale from given WR comparisons. We also show that response length variability alone can induce unequal optimal scales,
and that length normalization removes this effect in a controlled
setting. These results do not guarantee recovery of latent annotation noise.
\item We compare with DPO and related
methods on HelpSteer2, HelpSteer3 \citep{wang2025helpsteer3}, and
AlpacaEval \citep{alpaca_eval}. With Llama-3.1-8B-Instruct, \methodln{}
achieves Skywork-judged tie-adjusted win rates against matched DPO of
68.00\% and 65.31\% on the 400- and 800-prompt HelpSteer3 panels,
respectively, with higher mean rewards and shorter responses.
On AlpacaEval, using a GPT-4.1 judge and common GPT-4-Turbo reference
answers, its length-controlled win rate is 21.62\%, versus 16.39\% for
DPO and 15.30\% for SimPO.
\end{itemize}

\section{Preliminaries}
\label{sec:preliminaries}

\subsection{Preference data with strengths}

Let \(\Y\) and \(\X\) be the prompt and response spaces, respectively. We use
uppercase letters for random variables and lowercase letters for their
realizations. For a prompt \(Y\in\Y\) and two responses \(X_1,X_2\in\X\),
\textcolor{black}{the ordered comparison before annotation is \(W=(Y,X_1,X_2)\).}
A binary (non-tied) preference annotation supplies a direction \(S\in\{-1,1\}\), with
\(S=1\) indicating that \(X_1\) is preferred, and an ordinal strength
\(K\in\{1,\ldots,K_{\max}\}\), where larger values indicate stronger
preference. \textcolor{black}{Orienting the responses by the observed
direction gives \(Z=(Y,X^+,X^-,K)\),}
where \(X^+\) and \(X^-\) are the preferred and rejected responses.
Write \(z_i=(y_i,x_i^+,x_i^-,k_i)\), \(i=1,\ldots,N\), for the
training comparisons and \(\D\) for their uniform empirical distribution.

\textcolor{black}{We encode \emph{preference strength} using a fixed nonnegative,
nondecreasing map \(m:\{1,\ldots,K_{\max}\}\to[0,\infty)\)
and a fixed margin scaling coefficient \(\tau\geq0\).}
The default choice \(m(k)=k\) assigns equally spaced operational margins. %

\subsection{DPO reward-gap parameterization}

Direct preference optimization (DPO) uses a trainable policy \(\pi_\theta(\cdot\mid y)\) and a fixed reference policy \(\piref(\cdot\mid y)\), both with full support on $\X$ for every prompt $y \in \Y$. Define the sequence-level log ratio and its pairwise difference by
\begin{align}
g_\theta(y,x)
&=\log\frac{\pi_\theta(x\mid y)}{\piref(x\mid y)},\\
A_\theta(y,x,x')
&=g_\theta(y,x)-g_\theta(y,x').
\label{eq:dpo-gap}
\end{align}
For \(z=(y,x^+,x^-,k)\), we abbreviate
\(A_\theta(z)=A_\theta(y,x^+,x^-)\).

The DPO parameterization for a Kullback--Leibler (KL) penalty
\(\beta_0D_{\mathrm{KL}}(\pi_\theta(\cdot\mid y)\|\piref(\cdot\mid y))\)
with fixed \(\beta_0>0\) represents the implicit reward as
\(r_\theta(y,x)=\beta_0g_\theta(y,x)+b_\theta(y)\)
\citep{rafailov2023direct}.  The prompt-only offset \(b_\theta(y)\)
cancels in differences, giving
\begin{equation}
\Delta_\theta(y,x,x')=\beta_0 A_\theta(y,x,x').
\label{eq:implicit-reward-gap}
\end{equation}
The coefficient \(\beta_0\) maps policy log ratios to implicit reward units. The margin scaling coefficient \(\tau\) independently controls the strength margin.

Write \(\sigma(t)=(1+e^{-t})^{-1}\) for the logistic function and
\(\softplus(t)=\log(1+e^t)=-\log\sigma(-t)\) for the softplus function.
Under the unit-scale Bradley--Terry model, ordinary DPO minimizes
\begin{equation}
\mathcal L_{\mathrm{DPO}}(\theta)
=
\E_{Z\sim\D}
\left[
\softplus\!\left(-\beta_0 A_\theta(Z)\right)
\right].
\label{eq:dpo-loss}
\end{equation}
This objective uses the preference direction but not its strength.

\subsection{Heteroskedastic random utility}
\label{sec:rum-prelim}

To allow preference noise to vary across prompts, consider a
heteroskedastic random-utility model (RUM) with latent utilities
\begin{equation}
U_i=r^\star(Y,X_i)+q^\star(Y)\varepsilon_i,
\label{eq:rum}
\end{equation}
for \(i\in\{1,2\}\). Here, \(r^\star\) is a latent reward, \(q^\star(y)>0\) is a
prompt-dependent noise scale, and the errors
\(\varepsilon_1,\varepsilon_2\) are independent standard Gumbel variables conditional on \(W\).
Their difference is standard logistic, so
\begin{equation}
\Prob(U_1>U_2\mid W=(y,x_1,x_2))
=
\sigma\!\left(
\frac{r^\star(y,x_1)-r^\star(y,x_2)}{q^\star(y)}
\right).
\label{eq:heteroskedastic-bt}
\end{equation}
In this binary choice model, \(S=1\) when \(U_1>U_2\), and \(S=-1\)
otherwise. The resulting choice model is the heteroskedastic
Bradley--Terry model \citep{bradley1952rank}; it specifies neither ordinal
strength nor a tie category. The latent scale \(q^\star(y)\) controls choice
variability: for a fixed nonzero reward gap, increasing it moves the
preference probability toward \(1/2\). Binary probabilities identify only
the standardized reward gap. Multiplying \(r^\star(y,\cdot)\) and
\(q^\star(y)\) by the same positive prompt-dependent factor leaves these
probabilities unchanged.

\section{Methodology}
\label{sec:methodology}

Uncertainty-normalized margin DPO (\method{}) combines a strength-dependent
margin with a learned prompt scale. We first define two policy objectives,
\method{}-AO and \method{}-WR, which differ in where the margin enters the
normalization. We then describe how to fit the scale before policy
optimization and introduce \methodln{}, a length-normalized extension of WR.

\subsection{Strength-dependent margins and prompt scaling}
\label{sec:unm-dpo-objective}

\method{} uses a learned positive
relative prompt scale \(q_\psi:\Y\to(0,\infty)\), shared by all
comparisons with prompt \(y\). Both \(q_\psi\) and the latent RUM scale
\(q^\star\) are prompt-dependent, but the fitted scale is not assumed to
equal the latent noise scale. We fit \(q_\psi\) before training the final
policy and hold it fixed during policy optimization
(Section~\ref{sec:global-scale-fitting}).
For an oriented comparison
\(z=(y,x^+,x^-,k)\), define the AO and WR comparison scores
\begin{align}
\Gamma^{\mathrm{AO}}_{\theta,\psi}(z)
&=\frac{\beta_0 A_\theta(z)}{q_\psi(y)}-\tau m(k),
\label{eq:evidence-only}\\
\Gamma^{\mathrm{WR}}_{\theta,\psi}(z)
&=\frac{\beta_0 A_\theta(z)-\tau m(k)}{q_\psi(y)}.
\label{eq:whole-residual}
\end{align}
Each comparison score is the scalar input to the sigmoid in the policy
loss, not itself a probability or an external reward-model score.
The advantage-only formulation (\method{}-AO) divides the policy-induced
reward gap by \(q_\psi(y)\), leaving the strength margin outside the
normalization.  The whole-residual formulation (\method{}-WR) uses the
normalized margin residual: it subtracts the margin from the reward gap
before dividing by \(q_\psi(y)\). The coefficients
\(\beta_0>0\), \(\tau\geq0\), and the strength encoding \(m\) are fixed.

Given the fitted scale \(\widehat q=q_{\widehat\psi}\), each policy
minimizes its corresponding loss
\begin{equation}
\mathcal L_M(\theta;\widehat q)
=\E_{Z\sim\D}\!\left[
\softplus\!\left(-\Gamma^M_{\theta,\widehat\psi}(Z)\right)\right],
\qquad M\in\{\mathrm{AO},\mathrm{WR}\}.
\label{eq:unm-dpo-loss}
\end{equation}
The same scale-fitting objective supports both policy objectives.
At inference, responses are generated by the final policy alone;
the scale network is not needed.

\paragraph{Margin placement and interpretation.}
\label{sec:placement}
\label{sec:threshold-semantics}

The distinction is apparent from the reward gap required for a positive
comparison score:
\begin{equation}
\begin{aligned}
\Gamma^{\mathrm{AO}}_{\theta,\psi}(z)>0
&\iff \beta_0 A_\theta(z)>\tau m(k)q_\psi(y),\\
\Gamma^{\mathrm{WR}}_{\theta,\psi}(z)>0
&\iff \beta_0 A_\theta(z)>\tau m(k).
\end{aligned}
\label{eq:raw-gap-requirement}
\end{equation}
For a positive margin, increasing \(q_\psi(y)\) raises this threshold
under AO but leaves it unchanged under WR.  In WR, \(q_\psi(y)\)
instead controls how sharply the loss penalizes departures from the
threshold.  These are soft penalties, not hard constraints on the policy.
Both reduce to fixed-margin DPO when \(q\equiv1\), and to ordinary
DPO when additionally \(\tau=0\).

The RUM motivates division by a local scale but does not determine where
preference strength enters the objective. A utility difference exceeding
a threshold \(\tau m(k)\) in \emph{raw reward units} motivates WR;
a threshold in \emph{local noise units} motivates AO.
Appendix~\ref{app:rum} derives both cases. These motivating quantities are
probabilities of specified exceedance events, not
\(\Prob(S=s\mid W,K=k)\). We do not equate an observed strength category
with an exceedance event. The objectives are therefore
\emph{margin-exceedance surrogates}, not normalized ordinal likelihoods;
an exact category model also requires category boundaries and any tie or
selection mechanism (Appendix~\ref{app:ordinal}).

\subsection{Learning the prompt scale}
\label{sec:global-scale-fitting}

Since RLHF datasets often contain only a single response pair per prompt, jointly fitting the scale and optimizing policy presents significant challenges (see identifiability conditions in Theorem \ref{thm:wr-margin-identification}). To this end, we separate the scale fitting from the policy optimization, and we first fit the scale using log-ratio differences from freezed auxiliary policies,
which we call \emph{pilots}. Five prompt folds ensure that each comparison
is scored by a pilot that was not fine-tuned on its prompt.

Let \(\mathcal U=\{y_i:i=1,\ldots,N\}\) be the set of unique training
prompts.  A deterministic assignment
\(f:\mathcal U\to\{1,\ldots,5\}\) places all comparisons of each prompt
in the same fold.  For each fold \(j\), fine-tune a fixed-margin pilot
\(\pi_{\theta^{(-j)}}\) on the other four folds, using \(q\equiv1\)
and the same \(\beta_0,\tau,m\).
Let \(\mathcal I\subseteq\{1,\ldots,N\}\) index comparisons whose complete
prompt--response sequences can be scored without truncation under the
length limit specified in Appendix~\ref{app:scale-features}. For
\(i\in\mathcal I\), compute the scaled out-of-fold (OOF) log-ratio difference
\begin{equation}
\widetilde B_i
=\beta_0 A_{\theta^{(-f(y_i))}}(y_i,x_i^+,x_i^-),
\label{eq:oof-pilot-evidence}
\end{equation}
\textcolor{black}{using the same reference policy \(\piref\).
We pool these scaled log-ratio differences from all five held-out folds and,}
for a regularization coefficient \(\lambda_q\geq0\), fit one scale network by minimizing
\begin{equation}
\mathcal L_q^{\mathrm{WR}}(\psi)
=\frac{1}{|\mathcal I|}\sum_{i\in\mathcal I}
\softplus\!\left(
\frac{\tau m(k_i)-\widetilde B_i}{q_\psi(y_i)}\right)
+\frac{\lambda_q}{|\mathcal U|}\sum_{y\in\mathcal U}
[\log q_\psi(y)]^2.
\label{eq:global-wr-scale-fit}
\end{equation}
Only the scale parameters \(\psi\) are optimized in this step. The first term uses
the fixed WR residual \(\widetilde B_i-\tau m(k_i)\), and the second term penalizes
departures from \(q=1\).

\paragraph{Scale network, bounds, and normalization.}
\label{sec:centering}
\label{sec:parameterization}

We constrain the scale network so that its output is bounded, \(q_\psi(y)\in[1/2,2]\), and its geometric mean over unique training prompts equals one. We model the unconstrained scalar primitive as a small MLP that takes fixed prompt features and prompt-domain metadata as input. We construct $q_\psi$ from this primitive so that both constraints hold by construction (see Appendix~\ref{app:scale-parameterization}). The input consists of a 64-dimensional CountSketch of the frozen embedding of prompt $y$, together with prompt metadata such as the domain ID (see Appendix~\ref{app:scale-features}). We fit a single global network $q_\psi$ by minimizing the objective in \Eqref{eq:global-wr-scale-fit}.

After scale fitting, we freeze the network and normalization statistics and train separate AO and WR policies using \Eqref{eq:unm-dpo-loss}. Neither the pilots nor the scale network is required at inference. Appendix~\ref{app:scale-features} discusses the feature construction, regularization coefficient, and optimization settings, while Appendix~\ref{app:alternative-scale-learning} describes alternative estimators.

\paragraph{Interpretation of the fitted scale.}
\label{sec:interpretation}
The fitted \(\widehat q\) is a relative scale of pilot margin residuals. It may reflect annotation uncertainty, task difficulty, or unclear task specification \citep{zhang2025ICML}. However, it may also capture other sources of pilot error that have no clearly interpretable cause. Holding each prompt out of its pilot's fine-tuning data limits in-sample fitting of the log-ratio differences, but does not separate these sources or establish \(\widehat q=q^\star\). The scale plays a role analogous to a temperature \citep{guo2017calibration}. %
Appendix~\ref{app:scale-fitting-signal} analyzes the signed-residual fitting signal and the role of centering.

\subsection{Length-normalized WR}
\label{sec:length-normalized-wr}

To remove the impact of response length and learn the scale in the units of implicit reward per token, we additionally consider normalizing each response's reference-relative log-probability
ratio by its number of tokens. Let \(n(y,x)>0\) be the number of response
tokens whose log probabilities contribute to \(g_\theta(y,x)\), including
any scored end-of-turn token. For a fixed
coefficient \(\beta_{\mathrm{LN}}>0\), define
\begin{align}
A_\theta^{\mathrm{LN}}(z)
&=\frac{g_\theta(y,x^+)}{n(y,x^+)}
 -\frac{g_\theta(y,x^-)}{n(y,x^-)},\\
\Gamma_{\theta,\psi}^{\mathrm{LN\text{-}WR}}(z)
&=\frac{\beta_{\mathrm{LN}}A_\theta^{\mathrm{LN}}(z)-\tau m(k)}
{q_\psi(y)}.
\label{eq:length-normalized-wr}
\end{align}
Minimizing \(\E_{Z\sim\D}\softplus(-\Gamma^{\mathrm{LN\text{-}WR}}_{\theta,\widehat\psi}(Z))\)
with a frozen scale gives \emph{uncertainty- and length-normalized margin
DPO, whole-residual formulation} (\methodln). Unlike reference-free SimPO
\citep{meng2024simpo}, this objective retains the reference policy and the
strength-dependent margin.

For length-normalized scale fitting, retain in \(\mathcal I\) only
comparisons with nonempty scored responses. Evaluate the same fixed pilots
on these held-out comparisons using
\begin{equation}
\widetilde B_i^{\mathrm{LN}}
=\beta_{\mathrm{LN}}A_{\theta^{(-f(y_i))}}^{\mathrm{LN}}(z_i).
\label{eq:oof-pilot-evidence-ln}
\end{equation}
We replace \(\widetilde B_i\) by \(\widetilde B_i^{\mathrm{LN}}\) in
\Eqref{eq:global-wr-scale-fit}, fit \(q_\psi\), and freeze it before
final-policy training. Thus length normalization changes both the scale
fit and the final policy loss.
The coefficient \(\beta_{\mathrm{LN}}\) is distinct from the coefficient \(\beta_0\) as the former calibrates quantities measured per token while the latter per token-sequence (response).
Their numerical calibration does not imply equal KL regularization;
dataset-specific values are given in Appendices~\ref{app:main-experiment-details},
\ref{app:hs2-hpo} and~\ref{app:8b-sensitivity}.

\section{Theoretical analysis}
\label{sec:theoretical-analysis}

Our analysis addresses two questions relevant to prompt-scale learning: when known margins can separate the scale from reward differences in the WR comparison model, and how response-length variability can influence residual-based scale fitting. The first result characterizes a structural identification boundary; the second isolates a possible source of scale distortion. The practical estimator pools information across prompts, and its usefulness is assessed empirically.

\subsection{Structural identification by margin contrasts}
\label{sec:wr-identification}
\label{sec:identification}
\label{sec:relations}

We first ask whether the WR comparison scores in \Eqref{eq:whole-residual}
determine the prompt scale separately from the reward differences.
Here the comparison scores are
treated as given model quantities, not as observed preference directions
or strength labels. Without a margin, multiplying all reward differences
and the scale by the same positive constant leaves these scores unchanged.
Known margins in raw reward units can remove this ambiguity, but nonzero
margins alone are not sufficient.

\textcolor{black}{The WR model is related to the Rao--Kupper extension of
Bradley--Terry, which introduces a tie threshold while fixing the logistic
noise scale \citep{rao1967ties}. Here, the margins are known in fixed reward
units, and the prompt scale is unknown. The characterization below applies
standard comparison-graph arguments \citep{jiang2009statistical} to this
setting.}

\begin{theorem}[WR identification by margin contrasts]
\label{thm:wr-margin-identification}
Fix a prompt \(y\), a finite collection of comparisons
\(z_e=(y,x_e^+,x_e^-,k_e)\), and their known margins \(\tau m(k_e)\).
Allow arbitrary finite response rewards \(r_\theta(y,x)\), independent
of strength, and one free positive scale \(q_\psi(y)\) shared by all
comparisons. Write
\(\Delta_\theta(y,x,x')=r_\theta(y,x)-r_\theta(y,x')\), without
restricting these rewards to a policy parameterization, and suppose
\[
\Gamma^{\mathrm{WR}}_{\theta,\psi}(z_e)
=\frac{\Delta_\theta(y,x_e^+,x_e^-)-\tau m(k_e)}{q_\psi(y)}.
\]
The comparison scores uniquely determine the scale and every compared
reward difference if and only if there is no function \(h\) on responses
satisfying
\[
\tau m(k_e)=h(x_e^+)-h(x_e^-)
\qquad\text{for every comparison }e.
\]
\end{theorem}

The function \(h\) would assign one number to each response, regardless
of which comparison contains it. Thus the condition says that the known
margins cannot all be absorbed into response-specific reward offsets. If
such offsets exist, the scale can change without changing any WR
comparison score. If they do not, a combination of comparisons eliminates
the unknown reward differences while retaining a nonzero known margin,
which fixes the scale. For \methodln{}, the same sufficient condition applies to the response
rewards \(r_\theta^{\mathrm{LN}}(y,x)=\beta_{\mathrm{LN}}g_\theta(y,x)/n(y,x)\). Appendix~\ref{app:wr-identification-proof}
provides the proof, an illustrative example, and the detailed scope.

\begin{remark}
Theorem~\ref{thm:wr-margin-identification} is a property of the WR
comparison model, not a guarantee for the estimator in
\Eqref{eq:global-wr-scale-fit}.  Observed strength labels do not supply
the WR comparison scores, and most prompts in our training data have only one
distinct response pair (Appendix~\ref{app:comparison-support}). A single comparison between distinct responses cannot identify the prompt scale. Two or more comparisons for the same prompt can identify it if they satisfy the theorem’s margin-contrast condition. The theorem therefore cannot justify separate unrestricted scale estimates
for those prompts. Our estimator instead pools fixed OOF values to fit one shared
prompt-feature function, as described in Section~\ref{sec:global-scale-fitting}.
Appendix~\ref{app:round1-validation} separately audits whether
strength predicts held-out annotator agreement.
    
\end{remark}

\subsection{Effect of response length on scale fitting}
\label{sec:length-scale-analysis}

Scale fitting compares each pilot log-ratio difference with its strength
margin. The following construction holds the pilots, margins, normalized
log-ratio differences, and mean response length fixed, while changing only
response-length variability.

\begin{proposition}[Length variability can induce prompt-scale variation]
\label{prop:ln-wr-scale-distortion}
Consider a two-prompt population in which \(y_1\) and \(y_2\) each have
probability \(1/2\), and every comparison has the same margin
\(c=\tau m(k)>0\). Let \(n_0>0\) be their common mean response length
and \(V\) a random length fluctuation with the same distribution at both
prompts. To isolate one length factor, suppose both responses in a
comparison from \(y_j\) have token count
\[
L_j=n_0(1+\delta_jV),
\]
where \(0<\delta_2<\delta_1<1\), \(\E[V]=0\), \(|V|\leq1\), and
\(\Prob(V\neq0)>0\).
The mean length is \(n_0\) at both prompts, while \(y_1\) has greater
length variability because \(\delta_1>\delta_2\).

Let \(\widetilde B_j^{\mathrm{LN}}\) be the length-normalized scaled pilot
log-ratio difference from \Eqref{eq:oof-pilot-evidence-ln}, and suppose it
equals \(c\) for every comparison. Match the score coefficients at the mean
length by setting \(\beta_{\mathrm{LN}}=n_0\beta_0\). The corresponding
sequence-level difference from \Eqref{eq:oof-pilot-evidence} is then
\(\widetilde B_j=c(1+\delta_jV)\).

For \(\lambda_q>0\), minimize the expected objective in
\Eqref{eq:global-wr-scale-fit} over two freely varying prompt scales
\(q_j=q(y_j)\in[1/2,2]\), subject to the geometric-mean constraint
\(q_1q_2=1\). Sequence-level fitting has a unique minimizer satisfying
\(q_1>1>q_2\), whereas length-normalized fitting has the
unique solution \(q_1=q_2=1\).
\end{proposition}

Thus sequence-level fitting assigns a larger scale to the prompt with more
variable response lengths despite identical margins and normalized pilot
log-ratio differences. Length normalization removes this particular source
of scale variation. In this construction, the normalized data loss is
constant in the scales; the positive regularization penalty then uniquely
selects \(q_1=q_2=1\). The result does not guarantee noise recovery or
better generation quality.
Appendix~\ref{app:ln-wr-scale-distortion-proof} gives the proof and scope.

\section{Experiments}
\label{sec:experiments}

We compare \method{} and \methodln{} with DPO and related baselines.
The main experiments use HelpSteer3-Preference \citep{wang2025helpsteer3}
with Llama-3.2-1B-Instruct and Llama-3.1-8B-Instruct as initial policies.
Additional evaluations use HelpSteer2-Preference
\citep{wang2025helpsteer2preference} and AlpacaEval \citep{alpaca_eval}.

\subsection{Experimental setup}
\label{sec:experimental-setup}

\paragraph{Methods and training.}
The proposed methods are \method{}-AO, \method{}-WR, and the
length-normalized formulation \methodln{} from Section~\ref{sec:methodology}.
Each uses a WR-fitted global prompt scale, frozen before final-policy
training. We compare with DPO \citep{rafailov2023direct},
fixed-margin DPO, ODPO \citep{amini2024offset}, MMPO
\citep{kim2024mmpo}, \(\beta\)-DPO \citep{wu2024betadpo},
\(\gamma\)-PO applied to DPO \citep{sun2025dynamic}, LogNormal MixDPO
\citep{imai2026mixdpo}, SimPO \citep{meng2024simpo}, and SPO-basic
\citep{sharifnassab2024soft}. Here \emph{fixed-margin DPO} uses the fixed
strength-to-margin rule \(\tau m(k)\) with \(q\equiv1\), not a single
margin shared by all strength levels.

All main-table results are our reproductions, using seed 42 and the same
initial instruction-tuned model within each model-size setting. The main-table HelpSteer3 comparisons
use terminal 150-update final policies. HelpSteer2
budgets differ (Appendix~\ref{app:hs2-hpo}). Baseline settings are inherited from
earlier development or the literature, but they are not exhaustively retuned.
The proposed methods additionally require pilot training, so matched final-policy
budgets do not imply equal total compute. Optimizer settings and scale-fit
details are in Appendix~\ref{app:main-experiment-details}.

\paragraph{Evaluation and metrics.}
Policies generate greedily with at most 512 new tokens. The fixed
Skywork-Reward-V2-Llama-3.1-8B reward model scores candidate and DPO
responses separately. For each dataset and model, candidates on a given
panel share the same saved DPO responses. This trained DPO comparator is distinct from
the reference policy \(\piref\) used in the training objectives.
For \(n\) prompts, we report
\begin{equation}
\operatorname{Win}(\%)=100\frac{n_{\rm win}+n_{\rm tie}/2}{n},
\qquad
\Delta R=\frac1n\sum_{i=1}^{n}(R_i-R_{D,i}),
\label{eq:main-evaluation-metrics}
\end{equation}
where \(R_i\) and \(R_{D,i}\) are candidate and DPO scores. A candidate
wins when \(R_i-R_{D,i}>10^{-6}\); absolute differences at most \(10^{-6}\)
are ties. The length ratio \(L/L_D\) is mean candidate response length
divided by mean DPO response length, measured in characters. Win rates are
not length-controlled. Confidence intervals (CIs) are 95\% percentile
intervals from 10,000 paired prompt-bootstrap resamples, conditional on
the trained checkpoints and unadjusted for selection or multiple comparisons.

\paragraph{Evaluation panels.}
For HelpSteer3, the hyperparameter tuning was performed using a 500-prompt development
panel. We use the 400-prompt evaluation panel to examine ablations and model
variants, and the disjoint 800-prompt evaluation panel for confirmation
evaluation of the same frozen checkpoints.
The two evaluation panels were sampled with domain
stratification from the official validation split. We report the details in 
Appendix~\ref{app:original-protocol-contract}.

\subsection{HelpSteer3}
\label{sec:main-policy-results}
\label{sec:proposed-global-scale}

\begin{table}[t]
\centering
\resulttable
\renewcommand{\arraystretch}{1.0}
\setlength{\tabcolsep}{1.8pt}
\caption{HelpSteer3 with Llama-3.2-1B-Instruct on the 400- and 800-prompt
evaluation panels.}
\label{tab:main-global-comparison}
\begin{tabular*}{\linewidth}{@{\extracolsep{\fill}}>{\raggedright\arraybackslash}p{64pt}rrrrrr@{}}
\toprule
& \multicolumn{3}{c}{\textbf{400 prompts}} & \multicolumn{3}{c}{\textbf{800 prompts}} \\
\cmidrule(lr){2-4}\cmidrule(lr){5-7}
Method & Win (\%) \(\uparrow\) & \(\Delta R\;\uparrow\) & \(L/L_D\)
       & Win (\%) \(\uparrow\) & \(\Delta R\;\uparrow\) & \(L/L_D\) \\
\midrule
DPO (ref.) & 50.00 & 0.000 & 1.000 & 50.00 & 0.000 & 1.000 \\
\addlinespace[3pt]
Fixed-margin DPO & \winci{50.50}{45.63}{55.38} & \winci{+0.341}{-0.284}{+0.988} & 1.074 & \winci{50.62}{47.19}{53.94} & \winci{+0.216}{-0.281}{+0.716} & 1.094 \\
ODPO & \winci{53.88}{49.00}{58.63} & \winci{+0.529}{-0.092}{+1.140} & 1.056 & \winci{51.38}{47.94}{54.81} & \winci{+0.272}{-0.219}{+0.750} & 1.069 \\
MMPO & \winci{44.50}{39.63}{49.25} & \winci{-0.271}{-0.819}{+0.297} & 0.985 & \winci{49.13}{45.75}{52.44} & \winci{-0.289}{-0.703}{+0.119} & 0.981 \\
\(\beta\)-DPO & \winci{25.13}{21.00}{29.50} & \winci{-9.018}{-10.216}{-7.825} & 1.056 & \winci{24.50}{21.56}{27.56} & \winci{-8.606}{-9.410}{-7.777} & 1.050 \\
\(\gamma\)-PO (DPO) & \winci{47.25}{42.50}{52.00} & \winci{+0.119}{-0.396}{+0.636} & 0.981 & \winci{54.56}{51.31}{57.88} & \winci{+0.261}{-0.102}{+0.629} & 0.992 \\
LogNormal MixDPO & \winci{48.38}{43.63}{53.25} & \winci{-0.100}{-0.720}{+0.516} & 0.969 & \winci{49.38}{45.94}{52.88} & \winci{-0.450}{-0.915}{+0.016} & 0.963 \\
SimPO & \winci{54.13}{49.13}{59.00} & \winci{+0.287}{-0.553}{+1.084} & 0.952 & \winci{51.19}{47.69}{54.75} & \winci{+0.031}{-0.569}{+0.662} & 0.929 \\
SPO-basic & \winci{43.75}{39.00}{48.63} & \winci{-1.412}{-2.190}{-0.642} & 1.159 & \winci{43.31}{40.00}{46.75} & \winci{-1.701}{-2.286}{-1.123} & 1.166 \\
\midrule
\textbf{\method{}-AO} & \winci{58.25}{53.50}{63.00} & \winci{+1.040}{+0.384}{+1.693} & 1.064 & \winci{56.94}{53.56}{60.31} & \winci{+0.929}{+0.434}{+1.420} & 1.085 \\
\textbf{\method{}-WR} & \winci{58.75}{54.00}{63.50} & \winci{+1.071}{+0.364}{+1.788} & 1.073 & \winci{55.88}{52.44}{59.25} & \winci{+0.690}{+0.200}{+1.183} & 1.089 \\
\textbf{\methodln{}} & \winci{62.63}{57.88}{67.25} & \winci{+2.194}{+1.502}{+2.900} & 1.007 & \winci{61.50}{58.13}{64.81} & \winci{+1.803}{+1.310}{+2.296} & 0.998 \\
\bottomrule
\end{tabular*}
\tablenote{Brackets give conditional 95\% CIs for win rate and reward difference. Both panels
use the frozen judge, decoding and 10,000-resample paired-bootstrap protocol.}
\end{table}

\paragraph{Llama-3.2-1B-Instruct.}
{\color{black}
On both panels, all three proposed methods exceed the reproduced baselines
in win-rate and reward-difference point estimates
(Table~\ref{tab:main-global-comparison}). \methodln{} leads at 62.63\%
and 61.50\% wins on 400 and 800 prompts, respectively, with nearly DPO's
mean response length. All three methods' conditional CIs exclude
50\% wins and zero reward difference on both panels.
WR's point estimates exceed AO's on 400 prompts, but AO's exceed WR's
on 800, so their ordering is not established. Both sequence-level methods
produce modestly longer responses than DPO, which may contribute to
judge preferences.
}
\paragraph{Llama-3.1-8B-Instruct.}
{\color{black}
Table~\ref{tab:matched-8b} retains each selected checkpoint across panels
against the same 150-update 8B DPO comparator. \methodln{} leads in
win-rate and reward-difference point estimates, with shorter responses
than DPO. Selected during the 400-prompt variant evaluation, it achieves
68.00\% wins there and 65.31\% in 800-prompt confirmation, versus
SimPO's 61.00\% and 58.94\%, respectively. On the same 800-prompt panel, DPO and ULNM-DPO-WR also improve over the initial instruction-tuned model, achieving tie-adjusted win rates of 59.44\% and 70.13\% (Appendix \ref{app:initial_model_comparison}).
Appendix~\ref{app:alternative-scale-learning} reports additional scale
estimators, configurations, and training-budget results.

}

\begin{table}[t]
\centering
\resulttable
\renewcommand{\arraystretch}{1.0}
\setlength{\tabcolsep}{1.8pt}
\caption{HelpSteer3 with Llama-3.1-8B-Instruct on the 400- and 800-prompt
evaluation panels.}
\label{tab:matched-8b}
\begin{tabular*}{\linewidth}{@{\extracolsep{\fill}}>{\raggedright\arraybackslash}p{64pt}rrrrrr@{}}
\toprule
& \multicolumn{3}{c}{\textbf{400 prompts}} & \multicolumn{3}{c}{\textbf{800 prompts}} \\
\cmidrule(lr){2-4}\cmidrule(lr){5-7}
Method & Win (\%) \(\uparrow\) & \(\Delta R\;\uparrow\) & \(L/L_D\)
       & Win (\%) \(\uparrow\) & \(\Delta R\;\uparrow\) & \(L/L_D\) \\
\midrule
DPO (ref.) & 50.00 & 0.000 & 1.000 & 50.00 & 0.000 & 1.000 \\
\addlinespace[3pt]
Fixed-margin DPO & \winci{58.13}{53.25}{62.88} & \winci{+1.120}{+0.453}{+1.799} & 1.024 & \winci{54.75}{51.38}{58.19} & \winci{+0.322}{-0.174}{+0.816} & 1.004 \\
ODPO & \winci{55.50}{50.63}{60.25} & \winci{+1.042}{+0.412}{+1.692} & 1.012 & \winci{52.81}{49.38}{56.19} & \winci{+0.073}{-0.407}{+0.546} & 1.016 \\
MMPO & \winci{48.88}{44.13}{53.63} & \winci{-0.287}{-0.876}{+0.306} & 1.036 & \winci{45.75}{42.44}{49.19} & \winci{-0.691}{-1.111}{-0.266} & 1.023 \\
$\beta$-DPO & \winci{8.63}{6.00}{11.38} & \winci{-25.450}{-26.980}{-23.917} & 2.266 & \winci{8.69}{6.81}{10.69} & \winci{-25.533}{-26.604}{-24.434} & 2.246 \\
$\gamma$-PO (DPO) & \winci{53.63}{48.75}{58.38} & \winci{+0.423}{-0.171}{+1.019} & 1.021 & \winci{49.81}{46.44}{53.19} & \winci{+0.102}{-0.286}{+0.490} & 1.007 \\
LogNormal MixDPO & \winci{50.00}{45.13}{54.75} & \winci{+0.035}{-0.577}{+0.660} & 1.025 & \winci{48.06}{44.81}{51.44} & \winci{-0.384}{-0.819}{+0.051} & 1.016 \\
SimPO & \winci{61.00}{56.13}{65.63} & \winci{+1.702}{+0.819}{+2.571} & 0.988 & \winci{58.94}{55.56}{62.38} & \winci{+1.225}{+0.572}{+1.880} & 0.971 \\
SPO-basic & \winci{55.88}{51.00}{60.75} & \winci{+0.834}{+0.118}{+1.552} & 1.097 & \winci{54.19}{50.81}{57.56} & \winci{+0.599}{+0.081}{+1.117} & 1.093 \\
\midrule
\textbf{\method{}-AO} & \winci{54.50}{49.50}{59.38} & \winci{+0.667}{-0.024}{+1.367} & 1.036 & \winci{53.31}{49.87}{56.69} & \winci{+0.104}{-0.416}{+0.620} & 1.031 \\
\textbf{\method{}-WR} & \winci{59.38}{54.50}{64.25} & \winci{+1.245}{+0.511}{+1.978} & 1.025 & \winci{55.81}{52.38}{59.19} & \winci{+0.405}{-0.065}{+0.882} & 1.018 \\

\textbf{\methodln{}} & \winci{68.00}{63.38}{72.38} & \winci{+2.372}{+1.661}{+3.080} & 0.985 & \winci{65.31}{62.00}{68.56} & \winci{+1.796}{+1.241}{+2.356} & 0.964 \\
\bottomrule
\end{tabular*}
\end{table}

\subsection{Additional evaluations}
\label{sec:additional-evaluations}

\paragraph{AlpacaEval.}
We additionally evaluate the frozen 8B policies on 805 AlpacaEval
instructions, comparing each policy with GPT-4-Turbo reference answers
using a GPT-4.1 judge. \methodln{} attains a length-controlled win rate
of 21.62\%, compared with 16.39\% for DPO and 15.30\% for SimPO.
\textcolor{black}{Appendix~\ref{app:alpaca-8b} reports the complete LC win-rate
comparisons.}
\paragraph{HelpSteer2.}
\label{sec:helpsteer2-transfer}

On a 224-prompt evaluation panel
with Llama-3.2-1B-Instruct, \methodln{} attains 64.73\% wins against DPO
and \(\Delta R=+1.978\), compared with SimPO's 62.95\% and \(+1.749\).
These shared-reference results do not establish a significant difference
between the two methods. Table~\ref{tab:hs2-hpo-main} in
Appendix~\ref{app:hs2-hpo} gives the complete comparison, alongside
the  AO/WR sensitivity analysis.

To examine the contribution of learned prompt scaling, we compare \methodln{} against its constant-scale control, $(q\equiv1)$. 
On the 400- and 800-prompt HelpSteer3 panels, the learned-scale policy achieves 68.00\% and 65.31\% wins against DPO, respectively, compared with 62.00\% and 56.56\% for the control (Table \ref{tab:ulnm-8b-q1-control}). On the 800-prompt, \methodln{} achieves a direct tie-adjusted win rate of 57.69\% against the control (Appendix \ref{app:sec:constant-scale}).
On AlpacaEval, their length-controlled win rates against the common GPT-4-Turbo reference are 21.62\% and 17.83\%, respectively. These higher point estimates support the usefulness of learned scaling in the evaluated configurations. 

\FloatBarrier
\section{Conclusion}
\label{sec:conclusion}

We introduce \method{}, which combines strength-dependent margins with a learned prompt scale to adapt preference optimization to variation across prompts. A bounded scale network is fitted to pooled out-of-fold pilot log-ratio differences and then frozen during policy training. Its length-normalized extension, \methodln{}, additionally normalizes each response's reference-relative log-probability ratio by its token count. Our theoretical analysis characterizes how response length can affect scale fitting and when known margins identify the scale from given WR comparison scores, while distinguishing these results from recovery of latent annotation noise.

Experiments on HelpSteer2 and HelpSteer3 with the Skywork reward model, together with AlpacaEval evaluations using a GPT-4.1 judge, show improvements over DPO and several related baselines under the evaluated configurations. The strongest results are obtained with \methodln{}, supporting the combination of learned prompt scaling, strength-dependent margins, and length normalization. Fitting the prompt scale requires additional pilot training, whose cost grows with the number of folds. Our main comparisons retain the same final-policy update budget as DPO, and neither the pilots nor the scale network is required at inference. These results motivate further investigation of prompt-dependent scaling, including methods that reduce its training overhead and comparisons at matched total computational cost.

\bibliography{iclr2027_conference}
\bibliographystyle{iclr2027_conference}

\newpage
\appendix

\section{Random-Utility Derivations}
\label{app:rum}

Fix \(w=(y,x_1,x_2)\) and a candidate direction \(s\in\{-1,1\}\),
chosen before observing the utilities. Write
\(\Delta_s^\star(w)=s\{r^\star(y,x_1)-r^\star(y,x_2)\}\).
For a nonnegative raw-unit threshold \(c_k(y)\), define
\begin{equation}
E_{s,k}=\{s(U_1-U_2)>c_k(y)\}.
\label{eq:raw-threshold-event}
\end{equation}
The model in \Eqref{eq:rum} gives
\[
s(U_1-U_2)=\Delta_s^\star(w)+q^\star(y)\eta_s,
\qquad
\eta_s=s(\varepsilon_1-\varepsilon_2)\sim\operatorname{Logistic}(0,1).
\]
For any fixed nonnegative threshold \(c_k(y)\), symmetry of the logistic
CDF yields
\begin{align}
\Prob(E_{s,k}\mid W=w)
&=\Prob\!\left(\eta_s>
\frac{c_k(y)-\Delta_s^\star(w)}{q^\star(y)}\,\middle|\,W=w\right)\nonumber\\
&=\sigma\!\left(\frac{\Delta_s^\star(w)-c_k(y)}{q^\star(y)}\right).
\label{eq:raw-threshold-link}
\end{align}
Setting \(s=1\) and \(c_k(y)=0\) recovers the binary comparison model.
The raw threshold \(c_k(y)=\tau m(k)\) places both the reward gap and
the margin inside the scale normalization, motivating WR. Alternatively,
a threshold of \(\tau m(k)\) in local noise units sets
\begin{equation}
c_k(y)=\tau m(k)q^\star(y),
\label{eq:standardized-threshold}
\end{equation}
so that \(\Prob(E_{s,k}\mid W=w)
=\sigma(\Delta_s^\star(w)/q^\star(y)-\tau m(k))\), motivating AO.
Thus the same numerical encoding \(\tau m(k)\) has different units in
the two formulations: raw reward units for WR and local noise units for AO.
Substituting the
observed-orientation policy gap \(\beta_0A_\theta(z)\) and fitted scale
\(q_\psi(y)\) gives the surrogates in Section~\ref{sec:unm-dpo-objective}.
This substitution motivates the losses; it does not identify the observed
ordinal categories with the exceedance events. In particular, for a
positive threshold the two directional exceedance events exclude an
interval around zero and are not exhaustive. Their probabilities are not
\(\Prob(S=s\mid W,K=k)\). An exact ordinal likelihood requires the
category boundaries and any tie or selection mechanism, as in
Appendix~\ref{app:ordinal}.

\section{Scale Fitting and Structural Identification}
\label{app:unm-dpo-proofs}

Let \(\nu_{\mathcal U}\) be the uniform distribution over unique training
prompts. The scale constraints in Section~\ref{sec:centering} are
\begin{gather}
\E_{Y\sim\nu_{\mathcal U}}[\log q_\psi(Y)]=0,
\label{eq:geometric-centering}
\\
\frac{1}{2}\leq q_\psi(y)\leq2.
\label{eq:scale-bounds}
\end{gather}
These are optimization constraints, not guarantees of noise identification.

\subsection{Optimization rationale for scale learning}
\label{app:scale-fitting-signal}

The fitting signal is the discrepancy between the scaled OOF log-ratio
difference \(\widetilde B_i\) and the strength margin. Define the fixed residual
\begin{equation}
d_i=\widetilde B_i-\tau m(k_i),\qquad i\in\mathcal I.
\label{eq:pilot-margin-residual}
\end{equation}
For a single row, differentiation with respect to the log scale gives
\begin{equation}
\frac{\partial\,\softplus(-d_i/q)}{\partial\log q}
=\frac{d_i}{q}\,\sigma\!\left(-\frac{d_i}{q}\right).
\label{eq:scale-residual-gradient}
\end{equation}
Thus a positive residual favors a smaller scale, whereas a negative
residual favors a larger scale.  The shared network, centering, and
regularizer couple these row-level tendencies.  In particular, the scale
does not simply increase with the absolute residual or with preference
strength.

At \(q\equiv1\), an admissible infinitesimal log-scale direction \(v\)
must satisfy \(\E_{\nu_{\mathcal U}}v(Y)=0\). Its derivative in
\Eqref{eq:global-wr-scale-fit} is
\[
\frac{1}{|\mathcal I|}\sum_{i\in\mathcal I}
d_i\sigma(-d_i)v(y_i),
\]
because the quadratic log-scale penalty has zero derivative there.
The loss weights comparisons, whereas centering weights unique prompts;
only directions realizable by the shared network are available.
A negative derivative gives a local decrease, but a zero derivative
does not establish optimality. This finite-sample identity needs neither
independent residuals nor comparison cycles and also holds at \(\tau=0\).
Useful pooling requires predictive prompt features; the identity implies
neither noise recovery nor improved generation.

\paragraph{Why centering is not identification.}
Bounds and shrinkage prevent extreme scales but do not ensure convexity or
uniqueness.  For example, consider two prompts with one identical residual
\(d=0.55\) each and \(q_1=e^t,q_2=e^{-t}\).  With
\(\lambda_q=0.01\), the centered objective satisfies
\[
J''(0)=d\sigma(-d)\{d\sigma(d)-1\}+2\lambda_q<0.
\]
Thus even equal residuals need not make the constant scale a local minimum.
The same local direction is available through the centered-tanh map when
its scalar outputs can vary oppositely.  Learned scale variation alone
therefore does not establish underlying annotation-noise heterogeneity.

{
\subsection{Length-induced scale variation: proof and scope}
\label{app:ln-wr-scale-distortion-proof}

\begin{proof}[Proof of Proposition~\ref{prop:ln-wr-scale-distortion}]
\emph{1. Lengths determine the sequence-level residuals.}
For a comparison \((X_j^+,X_j^-)\) sampled at prompt \(y_j\), write
\(A_j=A_{\theta^{(-f(y_j))}}(y_j,X_j^+,X_j^-)\) for its fixed pilot's
unscaled log-ratio difference. Because both responses in this pair have
the same length \(L_j=n_0(1+\delta_jV)\),
\(\widetilde B_j^{\mathrm{LN}}
=\beta_{\mathrm{LN}}A_j/L_j\). With
\(\beta_{\mathrm{LN}}=n_0\beta_0\) and
\(\widetilde B_j^{\mathrm{LN}}=c\), this implies
\(\widetilde B_j=\beta_0A_j=c(1+\delta_jV)\).
The sequence-level residual \(\widetilde B_j-c\) therefore has the
distribution of \(c\delta_jV\); the length-normalized residual is zero.
This is the expected counterpart of \Eqref{eq:global-wr-scale-fit},
with equal prompt and comparison weights.
For the freely varying scales \(q_j=q(y_j)\), its sequence-level form is
\[
\mathcal J_{\mathrm{seq}}(q_1,q_2)
=\frac12\sum_{j=1}^2
\E\!\left[
\softplus\!\left(\frac{c-\widetilde B_j}{q_j}\right)
\right]
+\frac{\lambda_q}{2}\sum_{j=1}^2(\log q_j)^2.
\]

\emph{2. Reduce scale fitting to a strictly convex problem.}
Set \(s=\log q_1\). Centering gives \(q_2=e^{-s}\), and the scale bounds
are equivalent to \(s\in[-\log2,\log2]\). For length-normalized fitting,
\[
\mathcal J_{\mathrm{LN}}(e^s,e^{-s})=\log2+\lambda_qs^2.
\]
The data term is constant, so the positive penalty uniquely selects
\(s=0\), or \(q_1=q_2=1\).
For sequence-level fitting, use
\(\softplus(-u)=\log2+\log\cosh(u/2)-u/2\).
The linear term vanishes in expectation since \(\E[V]=0\);
symmetry of \(V\) is not required. Define
\[
H(v)=\E[\log\cosh(e^vV)],\qquad
a_j=\log(c\delta_j/2).
\]
Then
\[
F_{\lambda_q}(s)
:=\mathcal J_{\mathrm{seq}}(e^s,e^{-s})
=\log2+\tfrac12\{H(a_1-s)+H(a_2+s)\}+\lambda_qs^2.
\]
Boundedness of \(V\) permits differentiation under the expectation.
Writing \(t=e^vV\) inside that expectation gives
\[
H'(v)=\E[t\tanh t],\qquad
H''(v)=\E\!\left[t\tanh t+\frac{t^2}{\cosh^2t}\right]>0.
\]
Strict positivity follows from \(\Prob(V\neq0)>0\). Hence
\(F_{\lambda_q}\) is strictly convex and has a unique constrained minimizer.

\emph{3. Locate the minimizer and bound the scale ratio.}
Since \(a_1>a_2\) and \(H'\) is strictly increasing,
\[
F_{\lambda_q}'(0)=\tfrac12\{H'(a_2)-H'(a_1)\}<0.
\]
Let \(s_*=\tfrac12\log(\delta_1/\delta_2)>0\). At \(s_*\),
\(a_1-s_*=a_2+s_*\), so the two data-loss derivatives cancel:
\[
F_{\lambda_q}'(s_*)=2\lambda_qs_*>0.
\]
The unique minimizer of \(F_{\lambda_q}\) on \(\mathbb R\), denoted
\(r_{\lambda_q}\), thus lies strictly between \(0\) and \(s_*\).
Restricting it to the feasible interval gives
\(\widehat s=\min\{r_{\lambda_q},\log2\}\), so
\(0<\widehat s<s_*\). Substituting
\(q_1=e^{\widehat s},q_2=e^{-\widehat s}\) proves both
\(q_1>1>q_2\) and \(1<q_1/q_2<\delta_1/\delta_2\).
\end{proof}

\paragraph{Role of regularization.}
At \(\lambda_q=0\), the sequence-level optimum has
\(q_1/q_2=\min\{\delta_1/\delta_2,4\}\).
A positive penalty shrinks the unconstrained scale ratio toward one;
the constrained ratio may remain at its upper bound of four.
For length-normalized fitting, the data loss is
constant: the positive penalty selects unit scales, while every feasible
scale pair is optimal without it.

\paragraph{Connection to policy training.}
Let \(b=\beta_0A_\theta(z)\) for WR, or
\(b=\beta_{\mathrm{LN}}A_\theta^{\mathrm{LN}}(z)\) for length-normalized WR.
For a frozen scale \(q\),
\[
\left.\frac{\partial}{\partial b}
\softplus\!\left(\frac{c-b}{q}\right)\right|_{b=c}
=-\frac{1}{2q}.
\]
A larger scale therefore reduces the loss's sensitivity to \(b\) at
the margin. This is one factor in the policy gradient, not its full
parameter-space magnitude: the latter also depends on the gradient of
\(b\) itself.

\paragraph{Assumptions and scope.}
Using the same distribution of \(V\) at both prompts isolates a change in amplitude:
both prompts have mean response length \(n_0\), and their length standard
deviations have ratio \(\delta_1/\delta_2\), without requiring symmetry. Every
positive-probability outcome must yield an integer, without rounding
in the score identity.
An asymmetric example is \(n_0=8\), \((\delta_1,\delta_2)=(1/2,1/4)\),
and \(\Prob(V=-1/2)=2/3=1-\Prob(V=1)\).

Equal within-pair lengths, \(\widetilde B_j^{\mathrm{LN}}=c\), and the
coefficient matching are analytical assumptions, not empirical claims
about all training comparisons. In particular, being exactly at the margin
is essential to this result: it does not assert that greater length
variance always increases the fitted scale for arbitrary residuals.
At \(c=0\), both quantities vanish under this construction and this
particular effect disappears, without ruling out other length effects.

The optimization is over scale values, not network parameters. It does
not guarantee that a shared neural network realizes the minimizing pair
or that training finds it. The negative derivative at \(q_1=q_2=1\)
still gives a descent direction whenever the network can increase one
log scale and decrease the other by the same amount.
No latent noise model is assumed, and \(q_1=q_2=1\) is not identified
as the true annotation-noise scale. The proposition isolates a possible
source of scale variation; it establishes neither its prevalence in our
datasets nor a generation-quality guarantee.
}

\subsection{Proof of Theorem~\ref{thm:wr-margin-identification}}
\label{app:wr-identification-proof}

\begin{proof}
Fix \(y\) and abbreviate \(r(x)=r_\theta(y,x)\), \(q=q_\psi(y)\),
\(c_e=\tau m(k_e)\), and \(t_e=\Gamma^{\mathrm{WR}}_{\theta,\psi}(z_e)\).
Suppose \((r,q)\) and \((\widetilde r,\widetilde q)\) produce the same
scores, and set \(a=\widetilde q/q>0\). Equality of the scores gives
\[
\widetilde r(x_e^+)-\widetilde r(x_e^-)
=a\bigl[r(x_e^+)-r(x_e^-)\bigr]+(1-a)c_e
\qquad\text{for every }e.
\]
If \(a\neq1\), the function
\[
h(x)=\frac{\widetilde r(x)-a r(x)}{1-a}
\]
satisfies \(h(x_e^+)-h(x_e^-)=c_e\) for every comparison, contradicting
the stated condition. Hence \(a=1\), so \(\widetilde q=q\) and every
compared reward difference is the same under both parameter choices.

Conversely, suppose such a function \(h\) exists. For any \(a>0\), set
\[
\widetilde q=aq,\qquad
\widetilde r(x)=a r(x)+(1-a)h(x).
\]
Since \(h(x_e^+)-h(x_e^-)=c_e\), these parameters satisfy
\[
\frac{\widetilde r(x_e^+)-\widetilde r(x_e^-)-c_e}{\widetilde q}
=\frac{r(x_e^+)-r(x_e^-)-c_e}{q}=t_e.
\]
Taking \(a\neq1\) changes the scale without changing any score, proving
nonidentification of the scale and hence of the scale and compared
reward differences jointly.
\end{proof}

\paragraph{A three-response example.}
Suppose the same prompt has ordered comparisons \((a,b)\), \((b,c)\),
and \((a,c)\), with known margins \(c_{ab},c_{bc},c_{ac}\).
Any response rewards satisfy
\[
\Delta_\theta(y,a,b)+\Delta_\theta(y,b,c)-\Delta_\theta(y,a,c)=0.
\]
If \(c_{ab}+c_{bc}-c_{ac}\neq0\), the margins do not satisfy this
cancellation. The same combination of WR comparison scores therefore
equals \(-(c_{ab}+c_{bc}-c_{ac})/q_\psi(y)\), determining the scale.
This example illustrates the condition; the theorem does not require
this particular arrangement of comparisons.

\paragraph{Cycle interpretation and recovery.}
Fix a prompt \(y\). Represent its comparisons by a finite multigraph:
each distinct response is a vertex, and each comparison
\(z_e=(y,x_e^+,x_e^-,k_e)\) is an edge oriented from \(x_e^+\) to
\(x_e^-\). Multiple comparisons may connect the same responses.
Let \(n_v,n_e\) be the numbers of vertices and edges, and let
\(D\in\R^{n_e\times n_v}\) be the incidence matrix. Its row for edge
\(e\) is the coordinate vector of \(x_e^+\) minus that of \(x_e^-\);
it is therefore zero for a self-comparison \(x_e^+=x_e^-\).
Write \(r_i=r_\theta(y,x_i)\),
\(c_e=\tau m(k_e)\), \(t_e=\Gamma^{\mathrm{WR}}_{\theta,\psi}(z_e)\),
and \(q=q_\psi(y)\).  In this local vector notation,
\((Dr)_e=\Delta_\theta(y,x_e^+,x_e^-)\), so
\begin{equation}
t=\frac{Dr-c}{q}.
\label{eq:wr-graph-index}
\end{equation}
The condition in the body is exactly \(c\notin\operatorname{im}(D)\):
if \(c=Dh\), the entries of \(h\) assign one value to each response
whose differences reproduce every margin. We use the standard
comparison-graph decomposition \citep{jiang2009statistical} to express
this condition through cancellation of reward differences.

A cycle \(\mathcal C\) is a closed chain of distinct comparison edges
in the underlying undirected multigraph. An edge can be traversed in
either direction: set \(\omega_e=+1\) for traversal from \(x_e^+\)
to \(x_e^-\), and \(-1\) in reverse. Thus a cycle need not express
contradictory preferences. Parallel edges are allowed; a self-comparison
is a one-edge cycle with \(\omega_e=+1\).
For such a self-comparison, \(t_e=-c_e/q\);
if \(c_e\neq0\), its given comparison score alone determines
\(q=-c_e/t_e\). A zero-margin self-loop contributes no scale information.
This convention includes self-comparisons in the theorem and the argument
below.
The fundamental theorem of linear algebra gives
\(\operatorname{im}(D)^\perp=\ker(D^\top)\).  Therefore
\(c\notin\operatorname{im}(D)\) if and only if some
\(s\in\ker(D^\top)\) satisfies \(s^\top c\neq0\).
The kernel is spanned by cycle vectors: a cycle \(\mathcal C\) has
\(s_e=\omega_e\) on its edges and zero elsewhere.  Thus such an \(s\)
exists if and only if at least one cycle satisfies
\begin{equation}
\sum_{e\in\mathcal C}\omega_e\tau m(k_e)\neq0.
\label{eq:wr-margin-condition}
\end{equation}
For its cycle vector \(s\),
\begin{equation}
s^\top t=\frac{s^\top Dr-s^\top c}{q}
=-\frac{s^\top c}{q}.
\end{equation}
Its right-hand side is nonzero, giving \(q=-s^\top c/(s^\top t)\).
In the notation of the theorem, the scale and compared reward differences
are therefore recovered by
\begin{equation}
\begin{aligned}
q_\psi(y)&=-\frac{\sum_{e\in\mathcal C}\omega_e\tau m(k_e)}
{\sum_{e\in\mathcal C}\omega_e\Gamma^{\mathrm{WR}}_{\theta,\psi}(z_e)},\\
\Delta_\theta(y,x_e^+,x_e^-)
&=q_\psi(y)\Gamma^{\mathrm{WR}}_{\theta,\psi}(z_e)+\tau m(k_e).
\end{aligned}
\label{eq:wr-cycle-recovery}
\end{equation}
Equivalently, \(Dr=qt+c\). Any alternative parameters producing the same \(t\) must
therefore have the same scale and edge gaps.  Finally,
\(Dr=D\widetilde r\) implies that \(\widetilde r-r\) is constant on
each connected component.

\paragraph{Scope and model classes.}
At \(\tau=0\), the function \(h\equiv0\) satisfies the margin equations,
so identification fails in the unrestricted model.
Sufficiency is preserved under any restriction of the reward or scale
classes, including the policy-induced log-ratio rewards.  Necessity above
allows arbitrary node scores and a free positive scale at the fixed prompt.
Bounds, geometric centering across prompts, or architectural restrictions
may exclude some converse transformations.  The theorem does not characterize
necessity under those additional constraints.  It also requires fixed margins
in raw units and a common scale across each prompt's comparison edges.
The theorem concerns WR with a prompt-only scale, not AO or pair-specific
scales.

\paragraph{Strength contrasts and support.}
For a concrete example, fix \(y,x^+,x^-\) and consider
\(z_j=(y,x^+,x^-,k_j)\), \(j\in\{1,2\}\), with distinct raw margins
\(\tau m(k_1)\neq\tau m(k_2)\). The reward gap is the same in both
comparison scores because it does not depend on strength. Subtracting
the scores eliminates this gap, giving
\begin{equation}
q_\psi(y)
=\frac{\tau[m(k_2)-m(k_1)]}
{\Gamma^{\mathrm{WR}}_{\theta,\psi}(z_1)
-\Gamma^{\mathrm{WR}}_{\theta,\psi}(z_2)}.
\label{eq:wr-strength-contrast}
\end{equation}
The margin difference supplies a known quantity that determines the scale;
substitution into either score recovers the reward gap. These parallel
edges have cycle vector \(s=(1,-1)\), so the formula is a special case
of \Eqref{eq:wr-cycle-recovery}. On a triangle with edges
\((1,2),(2,3),(1,3)\), the condition is
\(c_{12}+c_{23}-c_{13}\neq0\).  Thus strength variation can provide an
anchor but is not necessary for every graph: equal nonzero margins satisfy
this triangle condition too.  On a forest, \(D\) is surjective onto the
edge space, so no margin vector can satisfy the condition.  No empirical
cycle-coverage assumption is asserted for the policy datasets.

\paragraph{Conditioning of inverse-scale recovery.}
Let \(P\) project orthogonally onto \(\ker(D^\top)\), and set
\(\kappa=\|Pc\|_2>0\) and \(\alpha=1/q\).
Projecting \Eqref{eq:wr-graph-index} gives \(Pt=-\alpha Pc\).
For perturbed scores \(\widehat t=t+\varepsilon\), define
\(\widehat\alpha=-(Pc)^\top\widehat t/\kappa^2\). Then
Cauchy--Schwarz gives
\begin{equation}
|\widehat\alpha-\alpha|
=\frac{|(Pc)^\top P\varepsilon|}{\kappa^2}
\leq\frac{\|P\varepsilon\|_2}{\kappa}.
\label{eq:wr-scale-stability}
\end{equation}
Nearly vanishing margin contrasts therefore make recovery ill-conditioned.
This deterministic comparison-score bound is not a finite-sample guarantee;
recovering those scores from probabilities can itself be unstable near
zero and one.

\subsection{Comparison support in the policy-training data}
\label{app:comparison-support}

We audited the frozen, non-tie training rows before the additional
pilot-scoring encoding and length mask. Exact prompt and unordered
response-pair identities were used, and repeated rows with the same orientation and strength were deduplicated for graph analysis.  A separate implementation reproduced the aggregate counts from the same checksum-bound training files. No validation, locked-test, or model-output records were used.

HelpSteer2 contains 6,766 rows over 6,765 prompts. Of these prompts,
6,764 have one distinct response pair and one has two disconnected pairs;
there are no ordinary response-graph cycles. HelpSteer3 contains 36,299
rows over 22,756 prompts: 22,112 have one distinct pair and 644 have
multiple pairs.  Only one prompt has an ordinary response-graph cycle.
Excluding self-comparisons, the nonzero margin-contrast condition holds
algebraically at zero HS2 prompts and 12 HS3 prompts. The latter consist
of 11 parallel-pair contrasts and one ordinary cycle. Each dataset also
contains one self-comparison, excluded from these non-self counts.
These records are self-loops under the convention in Appendix~\ref{app:wr-identification-proof}, not comparisons between distinct responses. The audit does not remove them from the policy-training data, and the counts above describe non-self comparison support only.

These are properties of the retained training representation, not counts
of all annotations in the official releases.  In particular, multiple
annotators judging the same pair do not create new response nodes, and
repeating an identical edge does not create a nonzero margin contrast.
Nor do realized labels reveal the conditional WR comparison scores assumed in
Theorem~\ref{thm:wr-margin-identification}.  The graph result therefore
cannot serve as an empirical noise-identification guarantee for these
experiments.  Our shared scale estimator instead pools residual information
across prompts through the fixed feature representation.

\section{Relation to an Exact Ordinal Model}
\label{app:ordinal}

An ordinal likelihood models the full signed annotation
\(C\in\{0,\ldots,J-1\}\), not only its unsigned strength \(K\).
For \(W=w=(y,x_1,x_2)\), let \(V=\Delta(w)+q(y)\xi\), where
\(\Delta\) is a finite location, \(q>0\) is a finite prompt scale,
and \(\xi\mid W=w\) has a common known, strictly increasing CDF \(F\).
These are auxiliary measurement-model parameters, not the fitted policy
gap and scale. Shared ordered cutpoints
\(-\infty=c_0<c_1<\cdots<c_{J-1}<c_J=+\infty\) define category \(j\)
by \(c_j<V\leq c_{j+1}\), giving
\begin{equation}
p_{\Delta,q,c}(j\mid W=w)
=F\!\left(\frac{c_{j+1}-\Delta(w)}{q(y)}\right)
-F\!\left(\frac{c_j-\Delta(w)}{q(y)}\right).
\label{eq:ordinal-pmf-appendix}
\end{equation}
These probabilities sum to one. The cutpoints are category boundaries,
not the operational margins \(\tau m(k)\). A tie category or selection
of non-tie observations must be represented when applicable; freely
prompt-dependent cutpoint spacings can absorb promptwise scale changes.
Our WR scale fit uses fixed pilot log-ratio differences and one observed
margin, not this full ordinal likelihood. The distinction supplies no
consistency or noise-recovery guarantee for the policy-training procedure.

\section{Audit of Preference Strength}
\label{app:round1-validation}

We test whether preference strength predicts held-out annotator agreement
in HelpSteer3. The audit's data partitions are distinct from the
policy-evaluation panels in Section~\ref{sec:experiments}. It does not
establish that the scale fitted in our policy experiments recovers true
annotation noise.

\subsection{Held-out-annotator audit of preference strength}
\label{app:strength-semantics}

\paragraph{Purpose and data.}
We test whether stronger judgments from some annotators predict greater
directional agreement with another annotator whose judgment is withheld.
This tests the information carried by strength labels, not their calibration
as confidence probabilities.

We deduplicate the union of the released HelpSteer3 training and validation
files, then hash full conversational contexts into a domain-stratified,
prompt-disjoint \(80/10/10\) partition (seed 260820): 19,816 training,
2,476 development, and 2,476 locked-test prompts. All records for a prompt
remain together; locked-test labels are excluded before parsing and are not
analyzed. The frozen source is \texttt{nvidia/HelpSteer3},
configuration \texttt{preference}, revision
\texttt{f6d145777bcbde96137596340fab89793acd1031}.

\paragraph{Leave-one-annotator-out design.}
For pair record \(j\), let \(n_j\) be its number of judgments and
\(V_{j\ell}\in\{-3,-2,-1,1,2,3\}\) the signed strength reported by
annotator \(\ell\). For each held-out annotator \(h\), retain the observation
only if \(n_j\geq3\) and all remaining annotators agree in direction.
Writing their common direction as \(d_{j,-h}\), define
\begin{equation}
e_{j,-h}
=
\left|\frac{1}{n_j-1}\sum_{\ell\neq h}V_{j\ell}\right|,
\qquad
I_{jh}
=
\mathbf{1}\!\left\{\operatorname{sign}(V_{jh})=d_{j,-h}\right\}.
\label{eq:strength-audit-statistic}
\end{equation}
Neither the evidence \(e_{j,-h}\) nor the reference direction uses the
held-out judgment. Prespecified evidence bins are low \([1,1.5]\),
middle \((1.5,2.5)\), and high \([2.5,3]\).

We estimate the high-minus-low agreement contrast using 5,000
prompt-clustered bootstrap replicates. A one-sided association test uses
10,000 permutations of record-level evidence within HelpSteer3 domains.
The prespecified success criterion requires positive association in both
training and development, development \(p\leq0.05\), at least 50 development
held-out votes in each extreme bin, and a positive development contrast
whose clustered 95\% interval excludes zero.

{\color{black}
For the permutation test, evidence and agreement are first averaged over
eligible held-out votes within each response-pair record. The statistic
is the pooled Pearson correlation between these record-level averages,
without domain centering; evidence is
permuted among records within each domain, with agreement held fixed.
The corrected one-sided value is \((1+b)/(10{,}001)\), where \(b\) counts
permutations whose correlation is at least the observed correlation
(numerical tolerance \(10^{-15}\)). Permutation seeds are 260823 and
260833 for training and development; the corresponding prompt-clustered
bootstrap seeds are 260822 and 260832. Records, rather than individual
votes, are the permutation units; prompts are the bootstrap clusters.
The permutation test assumes that record-level evidence is exchangeable
within each domain under the null. Domain-restricted permutations do not
preserve dependence between records sharing a prompt, so their \(p\)-value
is conditional on this additional assumption. The high-minus-low agreement
interval instead resamples whole prompts and accounts for within-prompt
dependence.
}

\paragraph{Results.}
The development partition contains 2,559 records from 2,476 prompts;
2,467 records yield 7,165 eligible held-out votes.
Table~\ref{tab:strength-audit} reports their binwise agreement.

\begin{table}[htbp]
\centering
\resulttable
\caption{Held-out-annotator directional agreement on the HelpSteer3
development partition. The last column reports descriptive Wilson intervals;
the high-minus-low contrast in the text uses prompt-clustered inference.}
\label{tab:strength-audit}
\begin{tabular*}{\linewidth}{@{\extracolsep{\fill}}lrrrr@{}}
\toprule
Evidence bin & Mean evidence & Held-out votes & Prompts & Agreement (95\% CI) \\
\midrule
Low \([1,1.5]\)       & \(1.229\) & \(3{,}236\) & \(1{,}390\) & \(96.35\%\;[95.65,96.95]\) \\
Middle \((1.5,2.5)\)  & \(2.000\) & \(1{,}864\) & \(1{,}070\) & \(100.00\%\;[99.79,100.00]\) \\
High \([2.5,3]\)      & \(2.804\) & \(2{,}065\) & \(894\)     & \(100.00\%\;[99.81,100.00]\) \\
\bottomrule
\end{tabular*}
\tablenote{Only observations with unanimous direction among the remaining
annotators are included. Prompts may contribute to more than one evidence bin.}
\end{table}

High-evidence agreement exceeds low-evidence agreement by 3.65 percentage
points (proportion difference \(0.03646\); clustered 95\% CI
\([0.02984,0.04315]\)). The pooled record-level evidence--agreement
correlation is \(0.2855\), with domain-restricted permutation
\(p=9.999\times10^{-5}\), the minimum
attainable corrected value for 10,000 permutations.
The training partition shows the same direction: 57,566 eligible votes,
correlation \(0.3027\), and contrast \(0.04199\) with clustered interval
\([0.03954,0.04451]\). These statistics meet the prespecified numerical
thresholds, subject to the permutation assumption above. The positive
prompt-clustered contrast provides the cluster-aware support for the
agreement result.
A secondary comparison reusing the annotations that formed the aggregate
label gives 100\% agreement at every non-tie strength; it is not an
independent check and is not used for the conclusion.

\paragraph{Interpretation.}
Stronger evidence predicts more reproducible directions \emph{conditional
on agreement among the remaining annotators}. The 100\% middle- and high-bin
rates therefore do not imply universal agreement in HelpSteer3.
The audit supports the relevance of strength labels, but establishes neither
equal cardinal spacing nor confidence calibration. It also cannot distinguish
a larger raw reward gap from a larger gap relative to noise: neither the
true gap nor the true prompt scale is observed. Thus it is consistent with,
but does not identify, the standardized-threshold model in
\Eqref{eq:standardized-threshold}. Any use of the real locked-test labels
requires a separately frozen confirmatory protocol.

\FloatBarrier
\section{Experimental Details}
\label{app:experimental-details}
\label{app:main-experiment-details}

This appendix specifies the training and evaluation settings for the main
comparisons. Appendix~\ref{app:alternative-scale-learning} describes alternative
scale estimators, Appendix~\ref{app:hyperparameter-sensitivity} reports
configuration sensitivity. All comparisons are against the
dataset-, model-, seed- and panel-specific DPO reference stated for the
experiment. Fixed-margin DPO is an additional baseline, not the reference
used to calculate win rates.

\subsection{Datasets, panels, and evaluation}
\label{app:original-protocol-contract}

The HelpSteer3 policy-training pool contains 36,299 eligible non-tie
comparisons; the official validation split contributes 1,920 eligible
non-tie rows. 
Earlier hyperparameter tuning used a 500-prompt development panel. We sampled 400- and 800-prompt evaluation panels ( using seeds 260834 and 260835) from the official validation split, ensuring that both were disjoint from each other and from the 500-prompt development panel.
The 400-prompt panel supports ablations and model-variant examinations;
the 800-prompt panel provides confirmation evaluation of frozen checkpoints.

HelpSteer2 contains 6,766 eligible training comparisons and 448 official
validation prompts. We sample a 224-prompt evaluation panel. 

{\color{black}
\paragraph{Dataset versions and preprocessing.}
We use the \texttt{preference} configurations of
\path{nvidia/HelpSteer3} at revision
\path{f6d145777bcbde96137596340fab89793acd1031} and
\path{nvidia/HelpSteer2} at revision
\path{990b2711a36180dd19d9c94b8627844866f8982a}.
The released signed label is \texttt{overall\_preference} for HS3 and
\texttt{preference\_strength} for HS2. Negative labels prefer response 1;
positive labels prefer response 2. We set \(k\) to the absolute label,
retain \(k\in\{1,2,3\}\), and exclude ties from preference training.
The data adapter stores normalized strength \(k/3\); the main margin and
scale-fitting objectives use \(k\), not \(k/3\).
Individual annotator judgments are not inputs to this main training pipeline.
Retained comparison records are not deduplicated.
HS3 has 38,459 training records before tie removal and 22,756 unique
prompts afterward; HS2 has 8,677 records before removal and 6,765 unique
training prompts afterward. Prompt-only HS2 evaluation retains all 448
validation prompts, including those whose original comparisons are tied.

\paragraph{Prompt identity and fold assignment.}
We preserve message contents without case folding or whitespace cleanup.
HS2 role markers are parsed into alternating user/assistant turns; an HS3
string prompt becomes one user turn, and existing conversation lists are
retained. Prompt identity is the SHA-256 hash of the UTF-8 JSON serialization
of the conversation, with sorted keys, unescaped Unicode, and default JSON
spaces. To assign a prompt to a fold, we hash
\texttt{identified-k1-fold-v1:260836:}\emph{prompt-hash}, interpret the
first eight digest bytes as an unsigned big-endian integer, and take its
remainder modulo five. All comparisons sharing a prompt therefore enter
the same fold. Each auxiliary policy is trained on the other four folds;
its held-out log-ratio differences contribute once to the pooled OOF values.
}

Generation is greedy with at most 512 new tokens. The judge is
Skywork-Reward-V2-Llama-3.1-8B, with its model revision fixed throughout.\footnote{Hugging Face model
\path{Skywork/Skywork-Reward-V2-Llama-3.1-8B}; revision
\texttt{cba2f842f3f1af2f1b2f0d35e794d789976390c5}.}
DPO responses are reused within each matched panel comparison.
The primary metrics are tie-adjusted win rate and mean candidate-minus-DPO
reward difference, as defined in \Eqref{eq:main-evaluation-metrics}.
A score difference of magnitude at most \(10^{-6}\) is a tie.
\(L/L_D\) is the ratio of mean response character counts; it is not a
length-controlled win rate. Confidence intervals are conditional 95\%
percentile intervals from 10,000 paired prompt-bootstrap resamples,
unadjusted for selection and multiple comparisons. Shared-reference
results do not constitute direct tests between candidate methods.

\subsection{Policy training and baseline settings}
\label{app:baseline-settings}

Comparisons share the same frozen initialization within each model setting:
\path{Llama-3.2-1B-Instruct} for 1B and
\path{Llama-3.1-8B-Instruct} for 8B. The common HelpSteer3 1B contract uses
seed 42, 150 final-policy updates, global batch size 128, microbatch size 1,
bfloat16, maximum sequence length 4,096, activation checkpointing, and
AdamW with learning rate \(10^{-6}\), weight decay 0.1,
\((\beta_1,\beta_2)=(0.9,0.98)\), \(\epsilon=10^{-5}\), and 20 warmup
updates. Reported configurations use terminal checkpoints; evaluations
of the same configuration on two panels use the same checkpoint.
HelpSteer2 budgets and deviations from these settings are stated below.
Pilot training adds compute beyond the final-policy budget.

{\color{black}
\paragraph{Initialization and optimization.}
The initial model snapshots use the following revisions:
\begin{table}[H]
\centering
\resulttable
\caption{Frozen initial-model revisions.}
\label{tab:model-revisions}
\begin{tabular*}{\linewidth}{@{\extracolsep{\fill}}ll@{}}
\toprule
Model & Frozen revision\\
\midrule
Llama-3.2-1B-Instruct & \texttt{9213176726f574b556790deb65791e0c5aa438b6}\\
Llama-3.1-8B-Instruct & \texttt{0e9e39f249a16976918f6564b8830bc894c89659}\\
\bottomrule
\end{tabular*}
\end{table}
For reference-relative objectives, \(\pi_{\rm ref}\) is a frozen copy of
this initial instruction-tuned model, not the trained DPO evaluation
opponent. Training updates all policy parameters; no LoRA adapters are
used. Gradient norms are clipped at one. Learning rates increase
linearly from 0.1 times the stated rate to that rate over 20 updates and
then remain constant. The loader traverses shuffled data without
replacement within an epoch, drops an incomplete final batch, and
reshuffles when another epoch is needed. Sequence packing and dynamic
batching are disabled. HS3 1B runs include intermediate validation-loss
checks, but the reported checkpoints are the terminal checkpoints at the
specified update budgets, not checkpoints selected by those losses.
}

\paragraph{Training tokenization and overlength pairs.}
Policy and auxiliary-policy training apply the 4,096-token limit separately
to each complete chat-formatted context--response sequence, including
template and termination tokens. There are no separate prompt and response
budgets. Both branches are retained without truncation when they fit.
If either branch exceeds the limit, the entire pair receives zero loss
weight; short dummy message prefixes replace both branches for batching,
rather than training on shortened responses. Masked pairs remain loader
and update-schedule slots but are excluded from the valid-pair loss
denominator.

For fixed-margin DPO and the sequence-level proposed objectives,
\(\beta_0=0.05\), \(\tau=1\), and \(m(k)=k\), except in the stated
margin-sensitivity study. Fixed-margin DPO uses \(q\equiv1\) with this
strength-dependent margin.

\paragraph{Reproduced baseline objectives.}
For comparison \(i\), write \(A_i=A_\theta(z_i)\) for the
reference-relative log-probability difference defined in the body.
Let \(P_i\) be the preferred-minus-rejected difference of policy
log-probability sums, without reference subtraction, and let \(\bar P_i\)
be the difference after dividing each response's sum by its own scored-token
count. All sums use the training mask above. Define binary logistic
cross-entropy by
\[
\operatorname{CE}(t,z)=-t\log\sigma(z)-(1-t)\log\sigma(-z).
\]
The reproduced per-pair losses are listed below; they are averaged over
valid pairs, except for the additional selection in \(\beta\)-DPO.
\begin{table}[H]
\centering
\resulttable
\caption{Implemented baseline losses for one preference pair.}
\label{tab:baseline-objectives}
\begin{tabular*}{\linewidth}{@{\extracolsep{\fill}}lp{0.71\linewidth}@{}}
\toprule
Method & Implemented per-pair loss\\
\midrule
DPO & \(\softplus(-\beta_0 A_i)\), with \(\beta_0=0.05\).\\
ODPO & \(\softplus(0.75k_i-\beta_0 A_i)\).\\
MMPO & \(\operatorname{CE}(\sigma(2.2k_i),\beta_0 A_i)\).\\
SimPO & \(\softplus(1.375-2.5\bar P_i)\).\\
SPO-basic & \(\softplus(-0.01P_i)/0.01\).\\
\(\beta\)-DPO & \(\softplus(-\beta_t A_i)\) on the selected batch subset.\\
\(\gamma\)-PO (DPO) & \(\softplus(\delta_i-\beta_0 A_i)\).\\
LogNormal MixDPO & \(-\log[16^{-1}\sum_{s=1}^{16}\sigma(b_{is}A_i)]\).\\
\bottomrule
\end{tabular*}
\end{table}
Here \(\beta_t\), \(\delta_i\), and \(b_{is}\) are defined below.
ODPO and MMPO use the raw strength \(k_i\), not the adapter's stored
\(k_i/3\). The other listed baselines use the observed preference direction
but do not use strength as a margin, weight, or target. SimPO is
reference-free and length-averaged, with learning rate \(10^{-6}\);
SPO-basic is reference-free and sequence-summed, with neither strength
weighting nor its optional global-output KL term. These losses have no
additional SFT term.

\paragraph{Batch-adaptive coefficients.}
All statistics below use detached pre-update log-probability differences
over the global batch of valid pairs, not individual microbatches.
For \(\beta\)-DPO, let \(\bar A_t\) and \(s_t\) be the batch mean and
sample standard deviation (zero for a singleton batch), and initialize
\(m_0=0,v_0=1\). Update
\[
m_t=0.9m_{t-1}+0.1\bar A_t,\qquad
v_t=0.9v_{t-1}+0.1s_t.
\]
From a batch of \(n\) valid pairs, sample
\(\max(1,\lfloor0.8n\rfloor)\) pairs without replacement, with weights
\(\exp\{-\tfrac12[(A_i-m_t)/\max(v_t,10^{-10})]^2\}\).
For the selected set \(S_t\), use
\[
\beta_t=\max\{10^{-3},\,0.05[1+0.2(\bar A_{S_t}-m_t)]\},
\]
where \(\bar A_{S_t}\) is its mean difference. Only selected pairs
contribute to the loss and its normalization. Validation-loss checks use
all valid pairs without updating these running statistics.

For \(\gamma\)-PO (DPO), combine the current scaled differences
\(a_j=\beta_0 A_j\) with a uniform sample without replacement from a
FIFO queue, up to 256 values in total. The queue holds at most 2,048
values and starts with eight independent \(\mathcal N(0,1)\) draws.
If the combined set has size \(n_c\), initialize \(p_j^{(0)}=1/n_c\).
With base margin \(\gamma_0=0.02\) and regularization coefficient 10,
the implemented mirror updates are, for \(r=0,\ldots,19\),
\begin{align*}
\gamma_j^{(r)}&=\gamma_0 n_c p_j^{(r)},\\
h_j^{(r)}&=10[1+\log(n_c(p_j^{(r)}+10^{-10}))]
 +\gamma_0\sigma(\gamma_j^{(r)}-a_j),\\
p_j^{(r+1)}&=\frac{p_j^{(r)}\exp(-0.05h_j^{(r)})}
{\sum_l p_l^{(r)}\exp(-0.05h_l^{(r)})}.
\end{align*}
The frozen implementation returns the current-batch components of
\(\delta_j=\gamma_j^{(19)}-\gamma_0\), computed before the last
probability update; it subtracts this deviation, not the raw
\(\gamma_j^{(19)}\), in the policy loss. The current scaled differences
are then appended to the queue during training; validation-loss checks do
not modify the queue. These coefficients are detached during
policy optimization.

For LogNormal MixDPO, draw independent \(\epsilon_{is}\sim\mathcal N(0,1)\)
and set \(b_{is}=\exp[\mu+\softplus(\rho)\epsilon_{is}]\).
The two global scalars \(\mu,\rho\) start at
\(\mu=\log(0.1)\), \(\softplus(\rho)=0.6\).
Once per batch, update them using the mixture loss above with detached
\(A_i\): AdamW with learning rate \(10^{-4}\), zero weight decay,
betas \((0.9,0.999)\), epsilon \(10^{-8}\), and gradient-norm clipping
at one. The policy update uses the same 16 coefficient draws sampled
before that auxiliary update, detached from \(\mu,\rho\).
The logarithm is outside the Monte Carlo average; validation-loss checks
also use 16 draws. For policy seed \(s\), the \(\gamma\)-queue initialization
and training-time sampling use Python random seeds \(s+91\) and \(s+92\);
the training-time \(\beta\)-subset and mixture draws use PyTorch seeds
\(s+93\) and \(s+501\), respectively. Validation-loss checks use separate
\(\gamma\)-sampling and mixture seeds \(s+502\) and \(s+1501\).

{
\paragraph{500-prompt hyperparameter tuning.}
At Llama-3.2-1B-Instruct tuning stage evaluated the five margin
candidates and comparator configurations in
Table~\ref{tab:early-500-tuning}. All candidates used policy seed 42 and
the same deterministic, domain-stratified 500-prompt development panel.
For each method, selection maximized mean paired Skywork reward difference
against the same seed-42 DPO policy. The margin sweep used the earlier AO
logistic argument \(\beta_0 A_\theta/q-\tau k\), with fixed
\(\beta_0=0.05\), unchanged bounds \(q\in[0.5,2]\), and the same
centered scale parameterization.

\begin{table}[H]
\centering
\resulttable
\caption{Hyperparameter search grids and selected values from the earlier
500-prompt tuning stage.}
\label{tab:early-500-tuning}
\begin{tabular*}{\linewidth}{@{\extracolsep{\fill}}p{0.28\linewidth}p{0.46\linewidth}p{0.22\linewidth}@{}}
\toprule
Method and parameter & Candidate values & Selected / fixed \\
\midrule
Earlier AO: margin \(\tau\)
& \(\{1/3,\,0.5,\,0.75,\,1,\,1.5\}\) & \(1\) \\
ODPO: strength multiplier
& \(\{1/3,\,0.5,\,0.75,\,1\}\) & \(0.75\) \\
MMPO: target slope \(\gamma\)
& \(\{0.3,\,0.7,\,1.1,\,2.2\}\) & \(2.2\) \\
\(\beta\)-DPO: adaptation coefficient \(\alpha\)
& \(\{0.2,\,0.6,\,1\}\) & \(0.2\) \\
\(\gamma\)-PO (DPO): base margin \(\gamma_0\)
& \(\{0.02,\,0.05,\,0.20,\,0.40\}\) & \(0.02\) \\
SimPO
& \(\beta\in\{2.5,10\}\), \(\gamma/\beta\in\{0.3,0.55\}\),\newline LR \(\in\{1,2\}\times10^{-6}\) & \(2.5\), \(0.55\), \(10^{-6}\) \\
SPO-basic & \(\alpha\in\{0.001,0.01\}\), LR \(\in\{2,5\}\times10^{-6}\)  & \(0.01\),\(5\times10^{-6}\) \\
UNM versions & \(\tau\in\{0.5,1,1.5\}\), LR \(\in\{1,2,5\}\times10^{-6}\)  & \(1\),\(10^{-6}\) \\
\bottomrule
\end{tabular*}

\end{table}

The MMPO target is \(\sigma(\gamma k_i)\); the \(\beta\)-DPO coefficient
\(\alpha\) occupies the position of \(0.2\) in the batch-adaptive
\(\beta_t\) formula above. Its retention rate remained \(0.8\), and the
\(\gamma\)-PO regularization coefficient remained \(10\). Moreover, we tested the best hyperparameters suggested by each method's own paper to be in a valid range.
}

{\color{black}
\paragraph{Auxiliary-policy training budgets.}
All five auxiliary policies in each dataset/model setting use the
unnormalized fixed-margin objective with \(\beta_0=0.05\), \(\tau=1\),
\(q=1\), learning rate \(10^{-6}\), and seed 42.
Both HS3 1B and HS2 1B use 150 updates per auxiliary policy.
HS3 8B uses one epoch over each four-fold complement, giving
227, 227, 227, 227, and 225 updates for folds 0--4.
HS2 8B uses two epochs, or 84 updates per auxiliary policy.
These budgets are separate from the final-policy budgets and from the
150 full-batch updates used to fit the scale network. Length-normalized
OOF rescoring reuses these same auxiliary policies.

Our methods require additional computation for auxiliary-policy training and scale fitting, adding training cost beyond the final-policy budget. However, these pilot checkpoints can be reused across scale-fitting and policy-objective variants within the same dataset and model setting. At deployment or test, only the final policy is needed; neither the pilots nor the scale network introduces additional inference cost. Moreover, the direct 800-prompt comparison supports the benefit of learned scaling: it achieves 57.69\% wins against the constant-scale control (95\% CI: [54.31, 61.00]).  
Using two or three folds instead of five would require fewer auxiliary policies and reduce pilot-training cost under the same per-pilot update budget. However, this potentially affects scale estimation and policy quality. We use five folds in the reported pilot-based experiments and leave this cost–quality tradeoff to future work.

}

\subsection{Scale fitting and features}
\label{app:scale-features}

The main sequence-level AO and WR methods use the WR scale-fitting
objective in \Eqref{eq:global-wr-scale-fit}. Five fixed-margin auxiliary
policies with distinct held-out prompt folds (fold seed 260836) supply
the pooled OOF values \(\widetilde B_i\). Each is trained on the other
four folds, so their training subsets overlap.
For HS3 1B, 35,945 rows are valid for pilot scoring; 354 are masked.
The data-fit term averages eligible rows, whereas normalization and
regularization average all 22,756 unique training prompts.
For HS3 1B, one prompt-only scale is fitted and frozen for both final policy losses.
The scale head is not evaluated during generation.

Prompt features are a 64-dimensional CountSketch of frozen prompt embeddings, standardized over unique training prompts.
{\color{black}The original HS3 1B head includes an additive learned scalar
domain term; HS2 1B and native 8B heads use a single constant domain ID.}
Width-16 GELU heads start at \(q=1\).
Response features, strength labels, and annotation entropy are not input
features of the main prompt head. Strength enters the fitting loss.
The optimizer is AdamW with learning rate \(10^{-3}\),
\((0.9,0.999)\) betas, \(\epsilon=10^{-8}\), gradient-norm clipping at 1,
and zero weight decay. Scale fitting uses seed 42, 150 full-batch updates,
\(\lambda_q=0.01\), and the centered parameterization below, with
\(q\in[0.5,2]\).
Only the pilot log-ratio differences are out of fold; the global scale
network is fitted to all pooled values. Architecture and calibration differences for
alternative estimators are given in Appendix~\ref{app:alternative-scale-learning}.

{\color{black}
\paragraph{Feature construction.}
The frozen encoder is the corresponding initial 1B or 8B model at the
revision specified above. We render its chat template with an assistant
prefix and fixed date \texttt{26 Jul 2024}, use left padding and left
truncation to 2,048 tokens, and embed batches of 32 in bfloat16.
The feature is the final-layer hidden state at the last prompt token,
converted to float32; its dimension is 2,048 for 1B and 4,096 for 8B.
For encoder coordinate \(j=0,\ldots,d-1\), CountSketch hashes
\texttt{pair-prompt-v1:}\(j\) with SHA-256. The first four bytes, read
big-endian modulo 64, give the bucket; the low bit of byte five selects
sign \(+1\) when set and \(-1\) otherwise. Signed bucket sums are divided
by \(\sqrt{\max(1,d/64)}\), where \(d\) is the encoder dimension.
Standardization uses featurewise population means and standard deviations
over unique training prompts, with standard deviations floored at
\(10^{-6}\); these statistics are frozen for evaluation.
For HS3 1B training prompts, the canonical domain is the modal training-record
domain, preferring records with nonempty responses, deduplicating only
for this domain vote, and breaking ties lexicographically; domain IDs
follow sorted domain names. Validation-only prompts use validation-record
domain metadata. HS3 1B \methodln{} reuses this feature/domain
mapping. The transferred 1B scale in the HS3 8B WR configuration also
retains its 1B encoder and domain mapping, rather than using 8B features.

\paragraph{OOF scoring mask.}
Auxiliary-policy and reference log probabilities are computed on complete
prompt--response sequences without truncation. A comparison is excluded
from scale fitting if either sequence exceeds 4,096 tokens. Prompt tokens
must form an exact prefix; only the response suffix, including its
template terminator but excluding padding, contributes to the advantage.
Length-normalized scoring also masks empty response suffixes and divides
each response's log-probability sum by its own suffix-token count.
Masked rows retain their identities and an explicit invalid flag; they
do not contribute to the likelihood. Their prompts remain in the
training-prompt population used for centering and regularization.
}

\paragraph{Bounded parameterization and centering.}
\label{app:scale-parameterization}
Let \(a_\psi(y)\) be the scalar network output and
\(\bar a_\psi=|\mathcal U|^{-1}\sum_{y\in\mathcal U}a_\psi(y)\).
For a bound parameter \(b>0\), define
\begin{align}
u_\psi(y)&=b\tanh\!\left(\frac{a_\psi(y)-\bar a_\psi}{b}\right),
\label{eq:bounded-raw-scale}\\
\bar u_\psi&=\frac{1}{|\mathcal U|}\sum_{y\in\mathcal U}u_\psi(y),
\label{eq:scale-center}\\
\log q_\psi(y)&=u_\psi(y)-\bar u_\psi.
\label{eq:scale-log-param}
\end{align}
At every scale-fitting update, we recompute both means over \(\mathcal U\)
and differentiate through them. This enforces
\Eqref{eq:geometric-centering} exactly. Since
\(u_\psi\) and \(\bar u_\psi\) lie in \([-b,b]\), setting
\(b=\log(2)/2\) guarantees the bounds \(q_\psi(y)\in[0.5,2]\).
For the sequence-level objectives, the coefficient on \(A_\theta\) is
\begin{equation}
\beta_{\mathrm{eff}}(y)=\frac{\beta_0}{q_\psi(y)},
\qquad
\E_{Y\sim\nu_{\mathcal U}}[\log\beta_{\mathrm{eff}}(Y)]=\log\beta_0.
\label{eq:effective-beta}
\end{equation}
It is determined by \(q_\psi\), not an additional learned parameter.
The same relation holds with \(\beta_{\mathrm{LN}}\) for the
length-normalized objective. After fitting, both means are frozen and
reused on new prompts without recentering. The bounds continue to hold,
but the geometric-mean constraint applies only to the training prompt
population. These are optimization constraints, not identification
guarantees for latent annotation noise.

\subsection{Length-normalized configurations}
\label{app:normalized-implementation}

For \methodln{}, the existing unnormalized fixed-margin pilots are rescored
on their held-out folds using reference-relative mean response-token
log probabilities. A fresh global scale is fitted with the WR objective
and then frozen. Pilot policies themselves are not length-normalized.
We determine \(\beta_{\mathrm{LN}}\) from training-response lengths, without
using evaluation-panel results. For each eligible training comparison
\(i\in\mathcal I\), let \(n(y_i,x_i^+)\) and \(n(y_i,x_i^-)\) count the
preferred and rejected response tokens contributing to their log-probability
ratios. These counts include the response terminator and exclude prompt
and padding tokens; comparisons with empty response suffixes or either
prompt--response sequence longer than 4,096 tokens are excluded.
Using the sequence-level coefficient \(\beta_0=0.05\), we set
\[
L_{\mathrm{ref}}
=\operatorname{median}_{i\in\mathcal I}
\left[\frac{n(y_i,x_i^+)+n(y_i,x_i^-)}{2}\right],
\qquad
\beta_{\mathrm{LN}}=\beta_0 L_{\mathrm{ref}}.
\]
Thus we average the two response lengths within each comparison, then take
the median across eligible training comparisons. This gives
\(\beta_{\mathrm{LN}}=0.05\times364.5=18.225\) for HS3 and
\(0.05\times272=13.6\) for HS2. Multiplication by \(L_{\mathrm{ref}}\)
compensates for dividing log-probability ratios by response length: the
sequence-level and normalized reward differences coincide when both
responses have exactly \(L_{\mathrm{ref}}\) tokens. For varying lengths,
this is only a magnitude calibration, not an equivalent KL penalty.
\begin{table}[H]
\centering
\resulttable
\caption{Training-length calibration of the normalized score coefficient.}
\label{tab:normalized-coefficients}
\begin{tabular*}{\linewidth}{@{\extracolsep{\fill}}lrr@{}}
\toprule
Setting & Median tokens & \(\beta_{\mathrm{LN}}\) \\
\midrule
HS3 1B & 364.5 & 18.225 \\
HS2 1B & 272 & 13.6 \\
HS3 8B & 364.5 & 18.225 \\
\bottomrule
\end{tabular*}
\end{table}
All three use seed 42 and \(\tau=1\).  The HS2 configuration retains the
learning rate and budget of the displayed sequence-level WR policy.
The HS3 1B scale fit uses the same 35,945 eligible OOF rows and 22,756
normalization prompts as its sequence-level counterparts, but fitting to
\(\widetilde B_i^{\mathrm{LN}}\) produces a distinct scale. Changing score normalization and
refitting the scale together does not isolate their separate effects.

\subsection{8B scale provenance and reference policies}
\label{app:8b-sensitivity}

In the main HS3 8B table, \method{}-AO uses a WR-fitted scale based on
five 8B pilots. \method{}-WR instead uses the frozen WR-fitted scale
from the HS3 1B pilots. Both are final 8B policies with prompt-only scale
inputs, learning rate \(10^{-6}\), seed 42, and 150 updates.
Their DPO reference also uses 150 updates. Separate policy and scale
networks do not make the fitted scale independent of the pilot policies.
The \methodln{} configuration uses a newly fitted scale from the
length-normalized OOF values \(\widetilde B_i^{\mathrm{LN}}\) of the 8B pilots.

\paragraph{Verification and reporting.}
Reported outputs undergo terminal integrity checks and independent
statistics replay before release. These checks validate computation,
not test-set independence. Win rates are shown to two decimals and
reward differences and length ratios to three; selection uses unrounded
values.

{\color{black}
\subsection{Software and compute allocation}
\label{app:software-compute}
The implementation uses NeMo-RL with PyTorch DTensor/FSDP2, Hugging Face
Transformers tokenization and model loading, and activation checkpointing.
The later checksum-pinned AO/WR, \methodln{}, and 8B requests use container
\path{nvcr.io/nvidia/nemo-rl@sha256:336aa41391a99e01d018d17d327107fd6d1023ad4b2812c8d8c913dee95fd3f2}.
The original HS3 1B baseline launcher defaults to
\path{nvcr.io/nvidia/nemo-rl:v0.6.0}; its frozen request does not bind a
container digest, so the later digest is not asserted for those earlier runs.
The selected 1B policy runs request one GPU on one node; selected 8B
policy runs request four GPUs on one node, with four-way data parallelism,
tensor parallelism one, and 32 gradient-accumulation microsteps at global
batch size 128. The one-GPU runs use 128 accumulation microsteps at the
same global batch and microbatch size one. Frozen requests record source,
model, data, feature, OOF-value, scale, and checkpoint checksums.
Evaluation verifies these bindings, exact prompt order and DPO-output
reuse, then independently recomputes summary statistics from terminal
outputs before admitting a result. These are integrity checks, not
additional model-selection or training steps.  The main training setup used NVIDIA A100-SXM4 GPUs with 80 GB of memory per GPU, allocating one GPU per 1B-model training job and four GPUs per 8B-model training job. An additional 8B sensitivity experiment used four NVIDIA H200 GPUs.
}

\FloatBarrier
\section{Alternative Scale-Learning Formulations}
\label{app:alternative-scale-learning}
\label{app:original-protocol-results}

These experiments compare ways to learn the scale in the AO and WR policy
losses, using HS3 and Llama-3.2-1B-Instruct. They include alternative
supervision, input features, normalization, and fitting schedules.
Because several ingredients change together, they are comparisons of
complete estimators, not controlled single-factor ablations.
Appendix~\ref{app:experimental-details} gives the shared evaluation protocol.

\subsection{Estimator definitions}
\label{app:original-protocol-selection}
\label{app:alternative-original-selection}

Let \(B_i=\beta_0 A_\theta(y_i,x_i^+,x_i^-)\) and \(c_i=\tau m(k_i)\),
as in the main objectives. The AO and WR logistic arguments are
\(B_i/q_i-c_i\) and \((B_i-c_i)/q_i\), respectively.
Write \(\chi_i\) for the scale-head input: \(\chi_i=y_i\) for a prompt-only
head and \(\chi_i=(y_i,\{x_i^+,x_i^-\})\) for a symmetric pair head;
\(f(i)\) denotes the prompt fold. Fixed-margin DPO sets \(q_i=1\).
Its margin is strength-dependent, not a constant shared by all pairs.

\begin{table}[H]
\centering
\resulttable
\caption{Scale estimators and their supervision and fitting schedules.}
\label{tab:scale-estimator-definitions}
\begin{tabular*}{\linewidth}{@{\extracolsep{\fill}}lp{0.62\linewidth}@{}}
\toprule
Scale estimator & Supervision and fitting schedule\\
\midrule
Pooled OOF, WR-fit & One prompt head fitted to all pooled OOF values \(\widetilde B_i\)
using the WR loss, then frozen; this is the main formulation.\\
Pooled OOF, AO-fit & One prompt head fitted to all pooled OOF values \(\widetilde B_i\)
using the AO loss, then frozen.\\
Cross-fitted, AO-fit & Five prompt or pair heads fitted on four folds of
OOF values \(\widetilde B_i\) each; held-out predictions are pooled and frozen.\\
Joint, AO-fit & A prompt or pair head fitted alongside the final policy
using detached current-policy values \(B_i\).\\
Entropy-CF & Five prompt or pair heads regress annotation dispersion on
four folds; pooled held-out predictions are transformed and frozen.\\
EMA-joint & A prompt or pair head uses current-policy values \(B_i\) and
moving-average centering during policy training.\\
\bottomrule
\end{tabular*}
\end{table}

\paragraph{AO-fitted scales.}
The cross-fitted and pooled AO-fit estimators minimize
\begin{equation}
 \mathcal L_q^{\mathrm{AO}}(\psi;\mathcal I,\mathcal U)
 =
 \frac{1}{|\mathcal I|}\sum_{i\in\mathcal I}
 \softplus\!\left(c_i-
 \frac{\widetilde B_i}{q_\psi(\chi_i)}\right)
 +\frac{0.01}{|\mathcal U|}\sum_{u\in\mathcal U}
 [\log q_\psi(u)]^2 .
 \label{eq:new-scale-q-fit}
\end{equation}
The data-fit term averages eligible comparison rows \(\mathcal I\);
centering and regularization average unique fit units \(\mathcal U\).
These are all unique training units for joint and pooled heads, and only
units outside fold \(f\), denoted \(\mathcal U_{-f}\), for cross-fitted head \(f\).
This is an AO scale-fitting objective even when the final policy uses WR.
Joint AO-fit heads use the same objective with each fixed OOF value
\(\widetilde B_i\) replaced by \(\operatorname{stopgrad}(B_i)\),
the detached current-policy scaled log-ratio difference. Each scale update uses the eligible current
batch for data fit and all unique training units for centering and
regularization. Policy updates detach \(q\). The AO and WR policies
therefore train separate heads whose realized scales need not coincide.

Frozen estimators use five fixed-margin pilots trained for 150 updates
on four of five prompt folds (fold seed 260836, policy seed 42).
Each comparison receives the scaled OOF log-ratio difference
\(\widetilde B_i=\beta_0 A_i^{(-f(i))}\) from the pilot excluding its fold.
Pooled fitting trains one scale head on all eligible OOF rows.
Cross-fitted scale head \(f\) instead fits rows outside fold \(f\), then
predicts that fold. Both use 150 full eligible-fit-row updates, not
150 epochs or policy-sized minibatches.
The architecture and optimizer follow Appendix~\ref{app:scale-features}.
Joint and pooled heads use seed 42; cross-fitted head \(f\), indexed
by \(f\in\{0,\ldots,4\}\), uses seed \(42+1009(f+1)\).

The WR-fit pooled estimator changes the scale-fitting loss to
\Eqref{eq:global-wr-scale-fit}. It reuses the same five pilots' values \(\widetilde B_i\),
35,945 valid OOF rows, and 22,756 normalization prompts.
Features, optimizer, initialization and the 150-update fit budget are
unchanged. One frozen WR-fitted head is shared by the two final
\method{} policies.

\paragraph{Pair features and normalization.}
Prompt inputs follow Appendix~\ref{app:scale-features}. Pair inputs have
388 dimensions: the prompt projection, symmetric sums, absolute differences
and products of 64-dimensional signed token-hash response vectors, two
prompt--response interaction vectors, and four length/similarity/overlap
scalars. These fixed response features are invariant to exchanging the
two responses. Domain contributes a learned additive scalar; neither
head receives winner indicators or strength as a feature.
Standardization uses all unique training units, including for
fold-specific heads; it is label-blind rather than refitted within each fold.
Width-16 GELU heads start at \(q=1\) and have no feature-perturbation penalty.

The parameterization in Equations~\ref{eq:bounded-raw-scale}--\ref{eq:scale-log-param},
with \(b=\log(2)/2\), applies to the appropriate prompt or pair inputs.
It gives geometric mean one on the fit population and \(q\in[0.5,2]\).
Each fitted head freezes its own fit-population centering means when
predicting held-out or evaluation inputs.  For each cross-fitted feature
scope separately, let \(\ell_{-f}(u)\) be that head's mapped log scale.
The final calibration is
\begin{equation}
 c_{\mathrm{OOF}}=\frac{1}{|\mathcal U_{\mathrm{all}}|}
 \sum_{u\in\mathcal U_{\mathrm{all}}}\ell_{-f(u)}(u),\qquad
 \widehat q_{-f(i)}(\chi_i)=
 \exp\!\left(\ell_{-f(i)}(\chi_i)-c_{\mathrm{OOF}}\right).
 \label{eq:new-scale-oof-calibration}
\end{equation}
Evaluation uses
\(\exp(\frac15\sum_{f=0}^4\ell_{-f}(u)-c_{\mathrm{OOF}})\).
Final bounds are checked after this single additive log offset; there
is no second tanh map.
These constraints define relative scale units, not proof that \(q\)
recovers true annotator noise rather than task difficulty or policy error.
There are 22,756 unique prompt units and 23,463 unique symmetric pair
units. These define the corresponding normalization populations.

The two-stage cross-fitting is not fully nested: a head excluding fold
\(f\) uses other rows' OOF values \(\widetilde B_i\) from pilots that may have trained
on \(f\). Standardization and final calibration also use pooled training
inputs or predictions. Thus the exclusions provide pilot-level OOF values
and direct scale-fit fold exclusion, not independence of the entire fitted
estimator from each held-out fold.

\subsection{Surrogate-scale comparisons}
\label{app:new-scale-confirmation}
\label{app:original-results-limitations}

Table~\ref{tab:original-selection-abbrev} compares joint, cross-fitted,
and pooled scale estimators on both HS3 1B evaluation panels.
Among the eight joint/cross-fitted configurations, the frozen rule required
positive mean reward difference and \(L/L_D\leq1.15\), then maximized mean
reward difference, with win rate as a tie-break. It selected joint prompt
WR. The four pooled configurations were outside this selection set;
WR-fitted scales were added after inspecting the earlier results.
The 800-prompt evaluation retains the same policies without reselection.

\begin{table}[H]
\centering
\resulttable
\caption{Surrogate scale-learning formulations on the HS3 1B 400- and
800-prompt evaluation panels. All policies use seed 42 and 150 updates.
Each panel uses outputs from the same DPO reference checkpoint.}
\label{tab:original-selection-abbrev}
\label{tab:new-scale-confirmation-42}
\setlength{\tabcolsep}{1.8pt}
\begin{tabular*}{\linewidth}{@{\extracolsep{\fill}}>{\raggedright\arraybackslash}p{74pt}rrrrrr@{}}
\toprule
& \multicolumn{3}{c}{\textbf{400 prompts}} & \multicolumn{3}{c}{\textbf{800 prompts}} \\
\cmidrule(lr){2-4}\cmidrule(lr){5-7}
Method & Win (\%) & \(\Delta R\) & \(L/L_D\)
       & Win (\%) & \(\Delta R\) & \(L/L_D\) \\
\midrule
Fixed-margin DPO & \winci{50.50}{45.63}{55.38} & \winci{+0.341}{-0.284}{+0.988} & 1.074
 & \winci{50.62}{47.19}{53.94} & \winci{+0.216}{-0.281}{+0.716} & 1.094 \\
\tablegroup{7}{Pooled prompt scale: AO-fitted \(q\)}
Pooled prompt AO (AO-fit) & \winci{57.00}{52.13}{61.75} & \winci{+1.173}{+0.536}{+1.825} & 1.072
 & \winci{55.06}{51.56}{58.44} & \winci{+0.666}{+0.159}{+1.170} & 1.082 \\
Pooled prompt WR (AO-fit) & \winci{54.50}{49.63}{59.38} & \winci{+0.818}{+0.155}{+1.462} & 1.085
 & \winci{55.38}{51.88}{58.81} & \winci{+0.633}{+0.115}{+1.150} & 1.089 \\
\tablegroup{7}{Pooled prompt scale: WR-fitted \(q\)}
\method{}-AO & \winci{58.25}{53.50}{63.00} & \winci{+1.040}{+0.384}{+1.693} & 1.064
 & \winci{56.94}{53.56}{60.31} & \winci{+0.929}{+0.434}{+1.420} & 1.085 \\
\method{}-WR & \winci{58.75}{54.00}{63.50} & \winci{+1.071}{+0.364}{+1.788} & 1.073
 & \winci{55.88}{52.44}{59.25} & \winci{+0.690}{+0.200}{+1.183} & 1.089 \\
\tablegroup{7}{Jointly learned scale}
Joint pair AO & \winci{56.50}{51.75}{61.25} & \winci{+1.048}{+0.362}{+1.735} & 1.073
 & \winci{52.31}{48.94}{55.75} & \winci{+0.383}{-0.118}{+0.877} & 1.084 \\
Joint pair WR & \winci{56.88}{52.00}{61.75} & \winci{+1.061}{+0.390}{+1.730} & 1.069
 & \winci{53.25}{49.81}{56.63} & \winci{+0.334}{-0.164}{+0.831} & 1.092 \\
Joint prompt AO & \winci{56.63}{51.75}{61.38} & \winci{+0.953}{+0.248}{+1.658} & 1.072
 & \winci{52.56}{49.13}{56.00} & \winci{+0.214}{-0.309}{+0.734} & 1.084 \\
Joint prompt WR\(^{*}\) & \winci{57.25}{52.38}{62.00} & \winci{+1.298}{+0.627}{+1.983} & 1.073
 & \winci{53.06}{49.69}{56.56} & \winci{+0.374}{-0.128}{+0.876} & 1.091 \\
\tablegroup{7}{Cross-fitted scale}
Cross-fitted pair AO & \winci{58.38}{53.63}{63.00} & \winci{+1.151}{+0.496}{+1.804} & 1.075
 & \winci{53.88}{50.38}{57.31} & \winci{+0.595}{+0.119}{+1.077} & 1.083 \\
Cross-fitted pair WR & \winci{57.88}{53.00}{62.63} & \winci{+1.195}{+0.512}{+1.865} & 1.072
 & \winci{53.81}{50.25}{57.31} & \winci{+0.566}{+0.077}{+1.061} & 1.085 \\
Cross-fitted prompt AO & \winci{57.13}{52.25}{61.88} & \winci{+1.060}{+0.395}{+1.719} & 1.059
 & \winci{53.19}{49.69}{56.69} & \winci{+0.601}{+0.108}{+1.087} & 1.078 \\
Cross-fitted prompt WR & \winci{55.63}{51.00}{60.38} & \winci{+1.056}{+0.364}{+1.746} & 1.057
 & \winci{51.88}{48.44}{55.38} & \winci{+0.541}{+0.049}{+1.043} & 1.072 \\
\bottomrule
\end{tabular*}
\tablenote{Brackets give conditional 95\% paired-bootstrap CIs; length ratios
are point estimates. The asterisk preserves the eight-arm 400-prompt selection.
}
\end{table}

\FloatBarrier
On the 800-prompt panel, reward-difference intervals include zero for the
joint configurations and are positive for the cross-fitted and pooled
configurations. These are comparisons against DPO, not direct tests among
estimators or evidence of annotation-noise recovery.

A direct 400-prompt comparison of WR-fit against AO-fit pooled scales
is inconclusive. For final WR policies, WR-fit obtains 51.75\%
[46.88, 56.50] wins and reward difference \(+0.254\)
[\(-0.285,+0.786\)] against AO-fit. For final AO policies, the corresponding
values are 48.75\% [43.88, 53.63] and \(-0.133\)
[\(-0.642,+0.366\)]. Both win-rate intervals include 50\% and both
reward-difference intervals include zero.

\FloatBarrier
\subsection{Entropy-supervised and EMA-centered scales}
\label{app:alternative-formulation-comparison}
\label{app:original-protocol-confirmation}

These earlier estimators differ from the current global-scale proposals;
their additional evaluations are not replications of those proposals.
The five-arm 400-prompt screen comprised EMA-joint pair AO and all four
Entropy-CF prompt/pair AO/WR configurations. The same positive-reward and
length-guard rule selected Entropy-CF prompt WR by maximum mean reward
difference; three subsequent EMA-joint additions did not change that selection.
Table~\ref{tab:alternative-scale-formulations} also includes the three
800-prompt comparisons at each of seeds 42 and 17, always against same-seed
DPO. This panel is disjoint from the 400-prompt.

\begin{table}[H]
\centering
\resulttable
\caption{Earlier entropy-supervised and EMA-centered HS3 1B estimators,
grouped by evaluation panel and policy seed. Every row uses same-seed DPO;
the seed-17 reference is independently trained.}
\label{tab:alternative-scale-formulations}
\label{tab:original-confirmation-42}
\label{tab:original-confirmation-17}
\begin{tabular*}{\linewidth}{@{\extracolsep{\fill}}lrrr@{}}
\toprule
Method & Win (\%) & \(\Delta R\) & \(L/L_D\) \\
\midrule
\multicolumn{4}{l}{\emph{400 prompts, seed 42}} \\
EMA-joint pair AO & \winci{52.75}{47.88}{57.63} & \(+0.534\) & --- \\
Entropy-CF pair AO & \winci{51.88}{47.13}{56.75} & \(+0.518\) & --- \\
Entropy-CF prompt AO & \winci{52.00}{47.12}{56.88} & \(+0.402\) & --- \\
Entropy-CF prompt WR\(^{*}\) & \winci{56.00}{51.13}{60.88} & \winci{+0.996}{+0.309}{+1.683} & 1.070 \\
Entropy-CF pair WR & \winci{48.25}{43.38}{53.00} & \(+0.186\) & --- \\
\addlinespace
EMA-joint prompt AO\(^{\dagger}\) & \winci{53.88}{49.00}{58.75} & \winci{+0.798}{+0.136}{+1.477} & 1.090 \\
EMA-joint prompt WR\(^{\dagger}\) & \winci{49.13}{44.25}{53.88} & \winci{+0.381}{-0.266}{+1.037} & 1.073 \\
EMA-joint pair WR\(^{\dagger}\) & \winci{53.25}{48.38}{58.13} & \winci{+0.633}{+0.011}{+1.275} & 1.063 \\
\midrule
\multicolumn{4}{l}{\emph{800 prompts, seed 42}} \\
EMA-joint prompt AO & \winci{54.25}{50.75}{57.69} & \winci{+0.411}{-0.091}{+0.911} & 1.097 \\
EMA-joint prompt WR & \winci{53.06}{49.62}{56.50} & \winci{+0.255}{-0.251}{+0.766} & 1.091 \\
Entropy-CF prompt WR & \winci{54.44}{50.94}{57.81} & \winci{+0.392}{-0.104}{+0.892} & 1.093 \\
\midrule
\multicolumn{4}{l}{\emph{800 prompts, seed 17}} \\
EMA-joint prompt AO & \winci{54.25}{50.81}{57.75} & \winci{+0.840}{+0.309}{+1.382} & 1.073 \\
EMA-joint prompt WR & \winci{53.19}{49.81}{56.63} & \winci{+0.864}{+0.330}{+1.416} & 1.069 \\
Entropy-CF prompt WR & \winci{55.13}{51.75}{58.50} & \winci{+0.827}{+0.301}{+1.359} & 1.059 \\
\bottomrule
\end{tabular*}
\tablenote{The asterisk marks the five-arm reward-based selection; daggers
mark subsequent descriptive additions. The 800-prompt EMA-joint rows are
additional comparisons, not new selections. Brackets are conditional 95\%
paired-prompt bootstrap intervals, unadjusted for selection or training-seed
uncertainty; length ratios are point estimates. Four 400-prompt rows retain
only reward point estimates: their omitted reward intervals and dashed
length cells denote unrecovered quantities, not unevaluated methods.
Point win rates retain the precision of the original result records.}
\end{table}

\paragraph{How Entropy-CF learns \(q\).}
For comparison \(i\), let \(n_{ic}\) count annotations in each of the six
signed-strength categories \(c\in\{-3,-2,-1,1,2,3\}\), and let
\(n_i=\sum_c n_{ic}\).  The Jeffreys-smoothed target is
\begin{equation}
 p_{ic}=\frac{n_{ic}+1/2}{n_i+3},\qquad
 u_i=\frac{\log 2}{\log 6}
 \left[-\sum_c p_{ic}\log p_{ic}\right]\in[0,\log 2].
 \label{eq:alternative-entropy-target}
\end{equation}
This is signed-strength dispersion for repeated judgments of one pair,
not merely directional disagreement or an observed latent noise scale;
it requires neither distinct pairs nor comparison cycles. Five
prompt-grouped folds fit \(h_{-f}\) outside fold \(f\) by minimizing
\begin{equation}
 \mathbb E_{i\notin f}
 \operatorname{SmoothL1}(h_{-f}(\chi_i),u_i)
 +0.01\,\mathcal L_{\mathrm{local\ smoothness}} .
 \label{eq:alternative-entropy-fit}
\end{equation}
The local penalty acts on raw predictions:
\[
 \mathcal S(g)=\mathbb E\!\left[
 \frac{(g(v+\epsilon,d)-g(v,d))^2}{0.01^2}\right],
 \quad \epsilon\sim\mathcal N(0,0.01^2I),\qquad
 \mathcal L_{\mathrm{local\ smoothness}}=\mathcal S(h).
\]
Only normalized features \(v\) are perturbed; domain \(d\) is fixed.
Width-16 GELU heads regress raw \(h\), not \(q\) or \(\log q\), for
30 epochs in batches of 512. Head \(f\in\{0,\ldots,4\}\) uses seed
\(260827+1009(f+1)\); deterministic prompt-hash folds differ from the
surrogate-fit seeded folds. Prompt inputs are embeddings and domain;
pair heads add symmetric response features, never winner/strength labels.
Both Entropy-CF and EMA-joint use AdamW with learning rate and weight decay
\(10^{-3}\), betas \((0.9,0.999)\), epsilon \(10^{-8}\), and
gradient-norm clipping at 1.

Pool held-out raw predictions \(\widehat u_i=h_{-f(i)}(\chi_i)\), apply
the centered--tanh-bounded--recentered map of
Equations~\ref{eq:bounded-raw-scale}--\ref{eq:scale-log-param}
\emph{once to that pooled vector} using comparison-row, not unique-unit,
means, and freeze \(q_i\). Evaluation averages the five raw head predictions
and uses the same training-derived transformation. No pilots or OOF policy
advantages enter this recipe. Prompt AO/WR share one frozen \(q\) artifact;
pair AO/WR share another; within each input class, only residual placement differs.

\paragraph{Out-of-fold predictive diagnostic.}
Across 36,299 retained HS3 comparisons, pooled Pearson correlations of
held-out raw predictions \(\widehat u_i\) with entropy targets \(u_i\)
are \(0.01543\) (prompt) and \(0.02715\) (pair): very weak linear
association, not a test of \(q\) against ground-truth noise or policy-evaluation
scores. The gate required positive correlation, nondegenerate scale
variation, and geometric centering; passing it establishes neither useful
accuracy nor significance. The transformation supplies no calibration in
noise-scale units, and these diagnostics do not establish noise recovery
or rule out learning disagreement with other features or estimators.
The main scales instead fit OOF log-ratio differences and aggregate
strength without entropy supervision.

\paragraph{How EMA-joint learns \(q\).}
Width-16 Tanh heads plus a domain term fit detached current-policy
scaled log-ratio differences with
\begin{equation}
 \log q_\psi(\chi_i)=\log 2\,
 \tanh\!\left(h_\psi(\chi_i)-c_{\mathrm{EMA}}\right),
 \qquad
 c_{\mathrm{EMA}}\leftarrow0.9c_{\mathrm{EMA}}
 +0.1\,\overline h_{\mathrm{batch}}.
 \label{eq:alternative-ema-q}
\end{equation}
Their scale-fitting surrogate matches the policy placement:
\(\softplus(c_i-\operatorname{stopgrad}(B_i)/q_i)\) for AO and
\(\softplus((c_i-\operatorname{stopgrad}(B_i))/q_i)\) for WR.
It adds \((\mathbb E_{\mathrm{batch}}\log q)^2\) and
\(0.01\mathcal S(\log q)\) with the perturbation rule above. Policy seed
\(s\) initializes the head at \(s+700\); head updates accompany all
150 policy updates, which detach \(q\). EMA centering bounds
\(q\in[0.5,2]\) without exact full-fit geometric mean one.

Unlike EMA-joint, the joint AO-fit methods above use GELU heads, exact
unique-unit centering, \(0.01\,\mathbb E_{\mathrm{unique}}[(\log q)^2]\),
zero head weight decay, no feature-perturbation penalty, and AO scale fitting
even for WR policies. Prompt standardization also changes from comparison-row
to unique-prompt weighting. These simultaneous differences preclude a clean
single-factor ablation.

Entropy-CF prompt WR improves over DPO on the 400-prompt panel despite weak
held-out entropy prediction; policy utility is not evidence of noise recovery.

\begingroup
\color{black}
\subsection{Centering and regularization ablation}
\label{app:scale-constraint-ablation}

We test the role of geometric centering and squared-log-scale regularization
in \methodln{} on HelpSteer3. Each model-size comparison is matched to
the corresponding \methodln{} configuration reported in the main text.
We reuse its auxiliary-policy outputs, refit the prompt scale, freeze it,
and train a fresh final policy. We compare the centered, regularized scale
with an uncentered scale retaining regularization and an uncentered scale
without regularization. All three retain \(q\in[0.5,2]\); neither uncentered
variant is an unbounded-scale model. The uncentered head uses
\(\log q=\log 2\,\tanh(h_\psi/\log 2)\), without training-set recentering.
Relative to the centered construction, this also changes the attainable
spread of scales, so the first comparison is not solely removal of a
mean constraint. The two uncentered variants use the same parameterization
and differ only in the scale regularizer.

\begin{table}[H]
\centering
\color{black}
\resulttable
\caption{Centering and regularization in \methodln{}, matched to the
corresponding main-text configurations. Win rates compare each policy
against the same-size DPO reference.}
\label{tab:ulnm-scale-constraints}
\begin{tabular*}{\linewidth}{@{\extracolsep{\fill}}lrrrr@{}}
\toprule
& \multicolumn{2}{c}{400 prompts} & \multicolumn{2}{c}{800 prompts} \\
\cmidrule(lr){2-3}\cmidrule(lr){4-5}
Scale fitting & Win (\%) & \(L/L_D\) & Win (\%) & \(L/L_D\) \\
\midrule
\multicolumn{5}{l}{\emph{Llama-3.2-1B-Instruct}} \\
Centered, regularized & \winci{62.63}{57.88}{67.25} & 1.007 & \winci{61.50}{58.13}{64.81} & 0.998 \\
Uncentered, regularized & \winci{64.13}{59.50}{68.88} & 1.024 & \winci{62.81}{59.50}{66.13} & 1.007 \\
Uncentered, unregularized & \winci{68.75}{64.25}{73.25} & 1.039 & \winci{61.00}{57.63}{64.38} & 1.024 \\
\midrule
\multicolumn{5}{l}{\emph{Llama-3.1-8B-Instruct}} \\
Centered, regularized & \winci{68.00}{63.38}{72.38} & 0.985 & \winci{65.31}{62.00}{68.56} & 0.964 \\
Uncentered, regularized & \winci{66.38}{61.75}{70.88} & 0.991 & \winci{65.38}{62.06}{68.69} & 0.970 \\
Uncentered, unregularized & \winci{69.00}{64.38}{73.38} & 0.987 & \winci{64.63}{61.25}{67.88} & 0.980 \\
\bottomrule
\end{tabular*}

\end{table}

Removing both centering and regularization gives the highest win-rate
point estimate on 400 prompts at both model sizes, but this advantage
does not persist on 800 prompts. Retaining regularization without centering
gives the highest 800-prompt point estimate, with a negligible difference
from the centered version at 8B. These results show no consistent advantage
from removing either constraint across panels; they do not establish
significant differences between variants. We also evaluate these scale-fitting alternatives for the main-text
1B \method{}-AO configuration (Table~\ref{tab:unm-ao-scale-constraints}).
The centered, regularized version has the highest win-rate point estimate
and the shortest responses on both panels. Removing centering while
retaining regularization increases the mean reward difference on 400
prompts, but not on 800 prompts.

\begingroup
\color{black}
\subsection{Constant-scale control}
\label{app:sec:constant-scale}
Table~\ref{tab:ulnm-8b-q1-control} compares a learned prompt scale with
\(q\equiv1\), retaining length normalization and strength-dependent margins.

\begin{table}[H]
\centering
\resulttable
\color{black}
\caption{Constant-scale control for \methodln{} on HelpSteer3 with
Llama-3.1-8B-Instruct on the 400- and 800-prompt panels.}
\label{tab:ulnm-8b-q1-control}
\renewcommand{\arraystretch}{1.0}
\setlength{\tabcolsep}{1.8pt}
\begin{tabular*}{\linewidth}{@{\extracolsep{\fill}}>{\raggedright\arraybackslash}p{64pt}rrrrrr@{}}
\toprule
& \multicolumn{3}{c}{\textbf{400 prompts}} & \multicolumn{3}{c}{\textbf{800 prompts}} \\
\cmidrule(lr){2-4}\cmidrule(lr){5-7}
Prompt scale & Win (\%) & \(\Delta R\) & \(L/L_D\)
             & Win (\%) & \(\Delta R\) & \(L/L_D\) \\
\midrule
Constant \(q=1\) & \winci{62.00}{57.25}{66.63} & \winci{+1.641}{+0.890}{+2.386} & 0.991
 & \winci{56.56}{53.13}{60.00} & \winci{+0.527}{-0.071}{+1.139} & 0.992 \\
Learned \(q\) & \winci{68.00}{63.38}{72.38} & \winci{+2.372}{+1.661}{+3.080} & 0.985
 & \winci{65.31}{62.00}{68.56} & \winci{+1.796}{+1.241}{+2.356} & 0.964 \\
\bottomrule
\end{tabular*}
\end{table}
\endgroup

\textbf{We additionally compare the learned-scale ULNM-DPO-WR policy directly against its constant-scale counterpart on the same 800-prompt HelpSteer3 panel. The learned-scale policy achieves a tie-adjusted win rate of 57.69\% (95\% CI: [54.31, 61.00]) and a mean Skywork reward difference of +1.270 (95\% CI: [+0.726, +1.822]). Both intervals exclude their respective null values. With length normalization and strength-dependent margins retained in both policies, this comparison supports the benefit of learned prompt scaling in the evaluated configuration.}

\begin{table}[H]
\centering
\color{black}
\resulttable
\caption{Centering and regularization in \method{}-AO on HelpSteer3
with Llama-3.2-1B-Instruct, matched to the main-text configuration.
All policies are evaluated against the same DPO reference.}
\label{tab:unm-ao-scale-constraints}
\begin{tabular*}{\linewidth}{@{\extracolsep{\fill}}llrrr@{}}
\toprule
Panel & Scale fitting & Win (\%) & \(\Delta R\) & \(L/L_D\) \\
\midrule
400 & Centered, regularized & \winci{58.25}{53.50}{63.00} & +1.040 & 1.064 \\
400 & Uncentered, regularized & \winci{57.13}{52.37}{62.00} & +1.233 & 1.103 \\
400 & Uncentered, unregularized & \winci{54.75}{49.88}{59.63} & +0.787 & 1.107 \\
\midrule
800 & Centered, regularized & \winci{56.94}{53.56}{60.31} & +0.929 & 1.085 \\
800 & Uncentered, regularized & \winci{56.56}{53.19}{59.94} & +0.598 & 1.117 \\
800 & Uncentered, unregularized & \winci{55.88}{52.50}{59.25} & +0.725 & 1.112 \\
\bottomrule
\end{tabular*}

\end{table}

\paragraph{Earlier joint scale saturation.}
An earlier sequence-level WR experiment learned the policy and prompt
scale jointly, with separate policy and scale updates on the same minibatches.
It used the uncentered bounded
sigmoid map \(q(y)=0.5+1.5\,\sigma(h_\psi(y))\), rather than the current
centered tanh construction, and retained a squared-log-scale penalty with
coefficient \(0.01\). Its fitted scales concentrated near the upper bound:
mean \(1.9917\) and standard deviation \(0.0149\). Most policy-minus-margin
residuals were negative; increasing their divisor reduced the WR fitting
loss, favoring a nearly uniform larger scale. In its earlier 500-prompt
comparison against the matched \(q=1\) WR baseline, the win rate was
\(52.60\%\;[48.20,57.00]\), with mean reward difference
\(+0.181\;[-0.363,+0.725]\), providing no clear evidence of improvement.
An additional joint AO experiment with this sigmoid-bounded head attained
\(45.10\%\;[40.90,49.40]\) against its matched \(q=1\) baseline;
its mean reward difference was \(-0.457\;[-1.023,+0.111]\).
Thus that earlier non-tanh variant performed worse by win rate, although
its reward-difference interval included zero.
These observations motivated controlling the global scale, but are not a
matched test showing that centering improves the present frozen-scale
method. In particular, this earlier non-tanh head was still bounded;
it does not demonstrate poor performance of a fully unrestricted scale.
\endgroup

\subsection{Comparison with the initial instruction-tuned model}
\label{app:initial_model_comparison}

We compare ULNM-DPO-WR, DPO, and SimPO directly against the
untouched Llama-3.1-8B-Instruct initialization on the same
800-prompt HelpSteer3 evaluation panel, using the Skywork judge.
Here, the evaluation opponent is the initial instruction-tuned
model. We report tie-adjusted win rate, mean candidate-minus-initial
reward difference $\Delta R$, and the ratio of mean response
character lengths. Confidence intervals are 95\% percentile
intervals from 10,000 paired prompt-bootstrap resamples,
conditional on the evaluated checkpoints.

\begin{table}[t]
    \centering
    \small
    \setlength{\tabcolsep}{5pt}
    \caption{Direct comparisons against the untouched
    Llama-3.1-8B-Instruct initialization on 800 HelpSteer3 prompts.
    Brackets indicate conditional 95\% confidence intervals.
    Length ratios are relative to the initial model.}
    \label{tab:initial_model_comparison}
    \begin{tabular}{lccc}
        \toprule
        Method & Win (\%) $\uparrow$ & $\Delta R$ $\uparrow$
        & Length ratio \\
        \midrule
        ULNM-DPO-WR
        & \textbf{70.13} & \textbf{+2.949} & 1.034 \\
        & [66.94, 73.31] & [2.388, 3.511] & \\[2pt]
        DPO
        & 59.44 & +1.152 & 1.072 \\
        & [56.19, 62.88] & [0.630, 1.679] & \\[2pt]
        SimPO
        & 62.94 & +2.377 & 1.042 \\
        & [59.62, 66.19] & [1.707, 3.034] & \\
        \bottomrule
    \end{tabular}
\end{table}

ULNM-DPO-WR improves over the initialization under this judge:
its win-rate interval excludes 50\%, and its reward-difference
interval excludes zero. DPO also improves over the initialization,
addressing the concern that ULNM's advantage over DPO primarily
reflects degradation of the DPO comparator. ULNM achieves these
gains with a smaller increase in mean response length than DPO
or SimPO. Together with the matched constant-scale comparison,
these results support the empirical value of the proposed
training procedure.

\FloatBarrier
\section{Hyperparameter and Training-Budget Sensitivity}
\label{app:hyperparameter-sensitivity}

This appendix reports configuration sensitivity under the shared evaluation
protocol in Appendix~\ref{app:experimental-details}. All comparisons use
seed 42. Win rates are tie-adjusted against the stated DPO reference;
confidence intervals are conditional on the trained policies and do not
adjust for configuration selection.

\subsection{HelpSteer2 1B: learning rate and training duration}
\label{app:hs2-hpo}

\paragraph{Policy comparison.}
Table~\ref{tab:hs2-hpo-main} reports comparisons on one panel against
the same original 104-update DPO responses. The related baselines retain
their method-specific settings and 104-update budgets; their full-panel
outputs are restricted to the same 224 indices before recomputing statistics.

\begin{table}[H]
\centering
\resulttable
\renewcommand{\arraystretch}{1.08}
\let\hswin\winci
\caption{HelpSteer2 with Llama-3.2-1B-Instruct on the
224-prompt   panel.}
\label{tab:hs2-hpo-main}
\begin{tabular*}{\linewidth}{@{\extracolsep{\fill}}lrrr@{}}
\toprule
Method & Win (\%) \(\uparrow\) & \(\Delta R\) (point) \(\uparrow\) & \(L/L_D\) \\
\midrule
DPO (reference) & 50.00 & 0.000 & 1.000 \\
Fixed-margin DPO & \hswin{52.23}{45.76}{58.71} & +0.825 & 1.023 \\
ODPO & \hswin{49.33}{42.63}{55.80} & +0.759 & 1.018 \\
MMPO & \hswin{51.56}{44.87}{58.04} & $-0.066$ & 0.965 \\
$\beta$-DPO & \hswin{57.59}{51.34}{63.84} & +1.173 & 1.016 \\
$\gamma$-PO (DPO) & \hswin{55.80}{49.33}{62.28} & +0.786 & 1.020 \\
LogNormal MixDPO & \hswin{53.57}{47.10}{60.27} & +0.794 & 1.026 \\
SimPO & \hswin{62.95}{56.69}{69.20} & +1.749 & 0.927 \\
SPO-basic & \hswin{37.72}{31.47}{44.20} & $-2.108$ & 1.203 \\
\midrule
\textbf{\method{}-AO} & \hswin{56.03}{49.55}{62.50} & +0.784 & 0.984 \\
\textbf{\method{}-WR} & \hswin{60.27}{54.01}{66.52} & +0.824 & 0.975 \\
\textbf{\methodln{}} & \hswin{64.73}{58.48}{70.98} & +1.978 & 0.959 \\
\bottomrule
\end{tabular*}
\end{table}

For the \methodln{} comparison in Table~\ref{tab:hs2-hpo-main},
the reward difference is \(+1.978\) [\(+1.042,+2.895\)].
Its length normalization and scale refitting are applied together
(Appendix~\ref{app:normalized-implementation}), so this comparison
does not isolate their individual contributions.

\paragraph{Learning-rate and training-duration sensitivity.}
Table~\ref{tab:hs2-hpo-all} reports the full twelve-configuration comparison
to characterize sensitivity to learning rate and training duration, rather
than rank or select configurations. For each of AO and WR, we vary the
learning rate over \(\{10^{-6},2\times10^{-6},5\times10^{-6}\}\) and
the training budget over 52 and 104 updates. All configurations share the
same frozen HelpSteer2 WR-fitted prompt scale, \(\beta_0=0.05\),
\(\tau=1\), and \(q\in[0.5,2]\), and use the same 224 prompts and
DPO reference as Table~\ref{tab:hs2-hpo-main}.
With 6,766 training comparisons, batch size 128 and \texttt{drop\_last},
52 updates constitute one loader epoch and 104 two epochs; the incomplete
110-row tail is omitted each epoch. The learning rate is constant after
20 warm-up updates.

\begin{table}[htbp]
\centering
\resulttable
\caption{Sensitivity to learning rate and training duration for HelpSteer2
with Llama-3.2-1B-Instruct on the 224-prompt panel. All other settings are
held fixed.}
\label{tab:hs2-hpo-all}
\begin{tabular*}{\linewidth}{@{\extracolsep{\fill}}lrrrrr@{}}
\toprule
Policy & Learning rate & Updates & Win (\%) & $\Delta R$ & $L/L_D$ \\
\midrule
AO & $10^{-6}$ & 52 & \winci{54.69}{48.21}{61.16} & +0.241 & 0.965 \\
AO & $2\times10^{-6}$ & 52 & \winci{56.92}{50.45}{63.17} & +1.041 & 0.967 \\
AO & $5\times10^{-6}$ & 52 & \winci{55.36}{48.66}{61.61} & +0.203 & 0.958 \\
AO & $10^{-6}$ & 104 & \winci{56.03}{49.55}{62.50} & +0.784 & 0.984 \\
AO & $2\times10^{-6}$ & 104 & \winci{55.80}{49.55}{62.05} & +0.883 & 0.974 \\
AO & $5\times10^{-6}$ & 104 & \winci{54.69}{48.21}{61.17} & +0.873 & 0.996 \\
\midrule
WR & $10^{-6}$ & 52 & \winci{56.03}{49.55}{62.28} & +0.099 & 0.977 \\
WR & $2\times10^{-6}$ & 52 & \winci{57.81}{51.34}{64.06} & +1.489 & 0.977 \\
WR & $5\times10^{-6}$ & 52 & \winci{50.00}{43.30}{56.25} & $-0.323$ & 0.951 \\
WR & $10^{-6}$ & 104 & \winci{60.27}{54.01}{66.52} & +0.824 & 0.975 \\
WR & $2\times10^{-6}$ & 104 & \winci{53.13}{46.43}{59.82} & +0.533 & 0.973 \\
WR & $5\times10^{-6}$ & 104 & \winci{57.14}{50.89}{63.39} & +0.972 & 0.995 \\
\bottomrule
\end{tabular*}
\end{table}
Eleven of twelve configurations have positive reward-difference point
estimates on this panel. 
Given the computational cost of training, our main comparisons use a fixed training seed (42). Appendix \ref{app:alternative-formulation-comparison} provides analysis of three earlier 1B entropy-supervised and EMA-centered estimators using seeds 17 and 42, each evaluated against a same-seed DPO baseline. Their win-rate point estimates are similar across the two seeds, providing evidence of stability for those configurations. Our paired-prompt bootstrap confidence intervals quantify evaluation-prompt uncertainty conditional on the trained checkpoints; broader replication across training seeds remains future work.

\FloatBarrier
\subsection{HelpSteer2 1B: margin sensitivity}

We vary only the final-policy
margin coefficient to \(\tau\in\{0.5,1.5\}\), retaining learning rate
\(10^{-6}\).
Table~\ref{tab:hs2-margin-sensitivity} includes the original
\(\tau=1\) configuration. The scale is not refitted.

\begin{table}[htbp]
\centering
\resulttable
\caption{HelpSteer2 WR margin sensitivity on the same 224-prompt panel. Learning rate, update budget and fitted scale are fixed.}
\label{tab:hs2-margin-sensitivity}
\begin{tabular*}{\linewidth}{@{\extracolsep{\fill}}lrrr@{}}
\toprule
$\tau$ & Win (\%) & $\Delta R$ & $L/L_D$ \\
\midrule
0.5 & \winci{54.46}{48.21}{60.71} & +0.639 & 0.979 \\
1.0 & \winci{60.27}{54.01}{66.52} & +0.824 & 0.975 \\
1.5 & \winci{54.24}{47.54}{60.49} & +0.717 & 0.989 \\
\bottomrule
\end{tabular*}
\tablenote{Same DPO reference and evaluation protocol as
Table~\ref{tab:hs2-hpo-all}. Brackets give conditional 95\% win-rate CIs.}
\end{table}

Both alternatives have lower point estimates than \(\tau=1\);
their reward-difference intervals are [\(-0.349,+1.615\)] and
[\(-0.292,+1.727\)] for \(\tau=0.5\) and \(1.5\), respectively.
These DPO-relative intervals do not test differences between margin
settings.

\clearpage
\section{AlpacaEval Comparison of HelpSteer3-Trained 8B Policies}
\label{app:alpaca-8b}

\paragraph{Setup.}
We evaluate the frozen Llama-3.1-8B-Instruct policies from
Table~\ref{tab:matched-8b}\textcolor{black}{, together with a \methodln{}
\(q\equiv1\) control}, each trained on HelpSteer3 for 150 final-policy
updates with seed 42, on the 805 AlpacaEval instructions
\citep{alpaca_eval}. Every policy, including DPO, is compared with the
same stored GPT-4-Turbo reference answers; DPO is not the opponent in
this evaluation. Candidate generation is greedy with at most 2,048 new
tokens, the saved native chat template, no input truncation, bfloat16
precision, batch size four, and generation seed zero.
These generation settings differ from the 512-token limit used for the
HelpSteer panels.

We use the \texttt{weighted\_alpaca\_eval\_gpt4.1} configuration with
judge snapshot \texttt{gpt-4.1-2025-04-14}, randomized answer order,
annotator seed zero, temperature one, and one output token with label
log probabilities. This replaces the original AlpacaEval~2 GPT-4-Turbo
judge, while retaining its reference answers and evaluation framework;
the scores are therefore not directly interchangeable with published
results using the original judge. The implementation is pinned to
AlpacaEval revision \path{cd543a149df89434d8a54582c0151c0b945c3d20}
and dataset revision \path{2edc6fad8be6b14ea7230aabfd08188da6b8b814}.

\paragraph{Metrics.}
\textcolor{black}{We report length-controlled (LC) win rate, which uses AlpacaEval's logistic adjustment}
for response-length differences \citep{dubois2024length}, evaluated at
equal candidate and reference lengths. We use
\texttt{length\_controlled\_v1}, five cross-validation splits,
random seed 123, and regularization coefficient 0.2.
\textcolor{black}{This metric differs} from the Skywork tie-adjusted win rate in
\Eqref{eq:main-evaluation-metrics}. All results passed independent
recomputation and output-checksum verification before reporting.

\begin{table}[H]
\centering
\resulttable
\caption{AlpacaEval results for HelpSteer3-trained Llama-3.1-8B-Instruct
policies, using GPT-4-Turbo reference answers and a GPT-4.1 judge.
All methods use the same 805 instructions.}
\label{tab:alpaca-8b}
\begin{tabular*}{\linewidth}{@{\extracolsep{\fill}}lrr@{}}
\toprule
Method & LC win (\%) $\uparrow$ & Mean characters \\
\midrule
DPO & 16.39 & 2148 \\
Fixed-margin DPO & 15.72 & 2348 \\
ODPO & 17.29 & 2284 \\
MMPO & 13.71 & 2492 \\
$\beta$-DPO & 0.49 & 8445 \\
$\gamma$-PO (DPO) & 14.80 & 2198 \\
LogNormal MixDPO & 13.67 & 2461 \\
SimPO & 15.30 & 2152 \\
SPO-basic & 17.44 & 3002 \\
\textcolor{black}{\methodln{} (\(q\equiv1\) control)} & \textcolor{black}{17.83} & \textcolor{black}{2374} \\
\midrule
\method{}-AO & 15.80 & 2446 \\
\method{}-WR & 16.98 & 2323 \\
\methodln{} & \textbf{21.62} & 2060 \\

\bottomrule
\end{tabular*}
\tablenote{\textcolor{black}{Bold marks the largest LC win-rate point estimate,
not statistical significance.} Mean characters measures generated
response length. Scores are against a common GPT-4-Turbo reference,
not direct comparisons between the trained policies.}
\end{table}

\paragraph{Results.}
\methodln{} has the highest LC win-rate point estimate, exceeding DPO
by 5.23 percentage points and SimPO by 6.32 points.
\textcolor{black}{The \(q\equiv1\) control achieves an LC win rate of 17.83\%.}
The LC result supports the usefulness of \methodln{} beyond the
HelpSteer prompt panels under a different judge.

\clearpage
\section{Related Work}
\label{app:related-work}

\paragraph{Preference optimization with strength information.}
DPO expresses preference learning through reference-relative response
log probabilities and a Bradley--Terry objective
\citep{rafailov2023direct}. ODPO incorporates comparison-dependent
offsets derived from response-quality differences
\citep{amini2024offset}, whereas MMPO converts quality margins into
soft preference targets \citep{kim2024mmpo}. HelpSteer2-Preference studies
margin-shifted and strength-weighted Bradley--Terry losses for reward
modeling and corresponding DPO variants
\citep{wang2025helpsteer2preference}. \method{} combines strength-dependent
margins with a learned prompt scale: AO normalizes the reward difference
before subtracting the margin, while WR normalizes the margin-subtracted
difference.

\paragraph{Adaptive objectives and preference heterogeneity.}
\(\beta\)-DPO adjusts a batch-level coefficient using reward-difference
statistics \citep{wu2024betadpo}, and \(\gamma\)-PO adapts comparison-specific
target margins during training \citep{sun2025dynamic}. MixDPO integrates
preference probabilities over a learned population-level distribution
of preference sensitivities \citep{imai2026mixdpo}; SPO instead develops
soft preference objectives for matching expert distributions
\citep{sharifnassab2024soft}. Our prompt-scale function is fitted to fixed
out-of-fold auxiliary-policy reward differences and strength margins,
then frozen before final-policy training. Disagreement can reflect task underspecification
or legitimate differences in preferences \citep{zhang2025ICML}; the
fitted scale is therefore not assumed to recover annotation noise.
Adaptive Preference Scaling \citep{hong2024adaptive} optimizes
comparison-specific scales during preference training. Its Ada-DPO
objective divides the preference gap by the scale and multiplies the
logistic loss by that same scale. The reported formulation uses neither
an explicit strength-dependent margin nor response-wise length
normalization. In one dialogue experiment, its quadratic-regularization
variant achieves a win rate of 56.00\%, compared with 53.38\% for DPO,
both evaluated against the same instruction-tuned reference
\citep[Appendix~C.1, Table~6]{hong2024adaptive}.

\paragraph{Length-normalized preference optimization.}
SimPO combines average response-token log probability with a constant
target margin, without reference-policy subtraction
\citep{meng2024simpo}. Appendix~H of that work also studies DPO with
response-wise length-normalized reference log-probability ratios.
Building on this objective, \methodln{} combines length normalization
with strength-dependent margins and learned prompt scales, using
normalized quantities in both scale fitting and final-policy training.
Our analysis isolates how response-length variation can affect scale
fitting even when normalized reward differences and preference strength
are unchanged.

\paragraph{Paired-comparison models and identification.}
Tie-aware and ordinal extensions of Bradley--Terry provide established
models for richer comparison outcomes
\citep{rao1967ties,agresti1992ordinal}. ODPO also gives a random-utility
threshold interpretation of margin shifts \citep{amini2024offset}.
Comparison-graph decompositions distinguish response-level differences
from cyclic components \citep{jiang2009statistical}. Building on this
structure, our theorem characterizes when known margins identify an
unknown prompt scale from given WR comparison scores. This structural
result is distinct from identifying latent noise from observed
preference labels.

\end{document}

%% file: math_commands.tex
\usepackage{amsmath,amsfonts,bm}

\def\eqref#1{equation~\ref{#1}}
\def\Eqref#1{Equation~\ref{#1}}

\def\1{\bm{1}}

\DeclareMathAlphabet{\mathsfit}{\encodingdefault}{\sfdefault}{m}{sl}
\SetMathAlphabet{\mathsfit}{bold}{\encodingdefault}{\sfdefault}{bx}{n}

\newcommand{\E}{\mathbb{E}}

\newcommand{\R}{\mathbb{R}}

\newcommand{\softplus}{\zeta}

